%% file: main.tex
\documentclass{article} 

\PassOptionsToPackage{table,dvipsnames}{xcolor}
\usepackage{iclr2027_conference,times}

\usepackage{amsmath,amssymb,amsthm,mathtools}
\usepackage{bm}
\usepackage{nicefrac}

\usepackage{graphicx}
\usepackage{caption}
\usepackage{subcaption}
\usepackage{wrapfig}

\usepackage{booktabs}
\usepackage{multirow}
\usepackage{makecell}
\usepackage{adjustbox}
\usepackage{array}
\usepackage{float}
\usepackage{tabularx}

\usepackage[table,dvipsnames]{xcolor}
\usepackage{colortbl}

\usepackage[ruled,lined]{algorithm2e}

\usepackage{microtype}
\usepackage{enumitem}
\usepackage{chngcntr}
\usepackage{placeins}
\usepackage{xspace}

\definecolor{inkblack}{HTML}{1C1F23}    
\definecolor{inkgraphite}{HTML}{3E464F} 
\definecolor{inkslate}{HTML}{5B6672}    

\definecolor{tabfocus}{HTML}{F4F4F2} 
\definecolor{tabrow}{HTML}{EAEAE7}   
\definecolor{tabhead}{HTML}{DCDCD8}  
\definecolor{tabbest}{HTML}{C9C9C4}  

\definecolor{inkrule}{HTML}{9A9A94}  
\definecolor{inkmuted}{HTML}{6E6E68} 

\usepackage[pagebackref=true,breaklinks=true,colorlinks,
            citecolor=inkgraphite,
            linkcolor=inkblack,
            urlcolor=inkslate]{hyperref}
\usepackage{aliascnt}
\usepackage{cleveref}
\usepackage{url}
\usepackage{etoolbox}

\makeatletter
\@ifundefined{Hy@theorem@makelinktarget}{}{%
  \newcommand{\paper@theorem@refstepcounter}[2][]{%
    \@skiphyperreftrue
    \if\relax\detokenize{#1}\relax
      \refstepcounter{#2}%
    \else
      \refstepcounter[#1]{#2}%
    \fi
    \@skiphyperreffalse
    \hyper@makecurrent{#2}%
    \global\let\Hy@dth@currentHref\@currentHref
    \Hy@theorem@makelinktarget{\Hy@dth@currentHref}%
  }
  \patchcmd{\cref@thmoptarg}{\refstepcounter}{\paper@theorem@refstepcounter}{}%
    {\PackageError{paper}{Cannot patch cleveref theorem anchors}{Check the cleveref/amsthm integration.}}
  \patchcmd{\cref@thmnoarg}{\refstepcounter}{\paper@theorem@refstepcounter}{}%
    {\PackageError{paper}{Cannot patch amsthm theorem anchors}{Check the cleveref/amsthm integration.}}
}
\makeatother

\theoremstyle{plain}
\newtheorem{theorem}{Theorem}[section]
\newaliascnt{proposition}{theorem}
\newtheorem{proposition}[proposition]{Proposition}
\aliascntresetthe{proposition}
\newaliascnt{lemma}{theorem}
\newtheorem{lemma}[lemma]{Lemma}
\aliascntresetthe{lemma}
\newaliascnt{corollary}{theorem}
\newtheorem{corollary}[corollary]{Corollary}
\aliascntresetthe{corollary}
\newtheorem*{theorem*}{Informal Theorem}

\theoremstyle{definition}
\newaliascnt{definition}{theorem}
\newtheorem{definition}[definition]{Definition}
\aliascntresetthe{definition}
\newaliascnt{assumption}{theorem}

\aliascntresetthe{assumption}
\newaliascnt{conjecture}{theorem}
\newtheorem{conjecture}[conjecture]{Conjecture}
\aliascntresetthe{conjecture}

\theoremstyle{remark}
\newaliascnt{remark}{theorem}
\newtheorem{remark}[remark]{Remark}
\aliascntresetthe{remark}

\input{macros}

\title{Mapping the Solution-Set Structure of Grokking on Latin Squares}

\author{
\textbf{Zihan Gu}$^{1,2}$, \textbf{Lulu Zhao}$^{3}$, \textbf{Tianyi Ma}$^{4}$,\textbf{Hua Zhang}$^{1,2}$ , \textbf{Wenbo Su}$^{3}$,\textbf{Yanan Cao}$^{1,2}$,\textbf{Yue Hu}$^{1,2}$ \\
  $^{1}$Institute of Information Engineering, Chinese Academy of Sciences; \\  $^{2}$School of Cyber Security, University of Chinese Academy of Sciences; \\ $^{3}$ Alibaba Group; $^{4}$Shanghai Jiaotong University. \\
  \texttt{\{guzihan\}@iie.ac.cn}
}

\iclrfinalcopy 

\begin{document}

\maketitle
\fancyhead{}

\begin{abstract}
Grokking is a delayed transition from memorisation to generalisation. A network first fits observed examples while remaining near chance on held-out inputs, then generalises abruptly. Prior work has focused mainly on this delay, finding that the network's main loss becomes weight decay after the fitting phase. We name the solution to this optimization problem as the minimum-norm endpoint, and study which tasks make all minimum-norm endpoints generalise from partial observations. Specifically, we model pairwise input tasks as edge colourings of $K_{q,q}$, with Latin colourings as our structured family, and classify each observed instance as all-good, mixed, or no-good by its endpoint set. Furthermore, we prove two contrasting results: (1) Finite-abelian-group families are asymptotically all-good from partial observations; (2) Under a uniform random Latin square, all-good tasks become asymptotically rare for observation-based learners, leaving mixed or no-good behaviour. Moreover, experiments show that relabellings of addition generalise despite their dense polynomial representations. Finally, on tasks of composite order, finite-training results suggest that predictors learn a coarse part of the output before the full label.
\end{abstract}


\input{section/01-introduction}
\input{section/03-related-work}
\input{section/02-model}
\input{section/03-overview}
\input{section/04-main-theorem}
\input{section/05-open-region}
\input{section/07-conclusion}


\input{main-bib}


\appendix
\crefalias{section}{appendix}
\crefalias{subsection}{subappendix}
\crefalias{subsubsection}{subsubappendix}
\input{appendix/00-notation}
\input{appendix/00-supplementary-theory}
\input{appendix/01-proofs}
\input{appendix/02-additional-experiments}
\input{appendix/03-discussion}

\end{document}

%% file: macros.tex
\newcommand{\Rbb}{\mathbb{R}}
\newcommand{\Zbb}{\mathbb{Z}}

\newcommand{\qset}{[q]}                        
\newcommand{\Xq}{X_q}                          
\newcommand{\Cq}{\mathcal{C}_q}                
\newcommand{\Lq}{\mathcal{L}_q}                
\newcommand{\Qallq}{\widehat{\mathcal{Q}}_q}   
\newcommand{\Qq}{\mathcal{Q}_q}                
\newcommand{\Sq}{\mathcal{S}_q}                
\newcommand{\Dset}{\mathcal{D}}                

\newcommand{\Prob}{\operatorname{\mathbb{P}}}
\newcommand{\Exp}{\operatorname{\mathbb{E}}}
\newcommand{\Unif}{\operatorname{Unif}}
\newcommand{\simplex}{\Delta_2}

\DeclareMathOperator{\relu}{ReLU}
\DeclareMathOperator{\rank}{rank}
\DeclareMathOperator{\per}{per}
\DeclareMathOperator*{\argmin}{arg\,min}

\newcommand{\opnorm}[1]{\lVert #1 \rVert_{\mathrm{op}}}
\newcommand{\nucnorm}[1]{\lVert #1 \rVert_{*}}
\newcommand{\frob}[1]{\lVert #1 \rVert_{F}}
\newcommand{\euc}[1]{\lVert #1 \rVert_{2}}
\newcommand{\card}[1]{\lvert #1 \rvert}
\newcommand{\pospart}[1]{\left( #1 \right)_{+}}

\newcommand{\Acc}{\operatorname{Acc}}
\newcommand{\Cadd}{C_q^{\mathrm{add}}}         
\newcommand{\Ccar}{F^{\star}_{q}}              
\newcommand{\Ccarc}[1]{F^{\star}_{q,#1}}       
\newcommand{\Madd}{\mathsf{M}_{\mathrm{add}}}
\newcommand{\Munif}{\mathsf{M}_{\mathrm{unif}}}
\newcommand{\Aplus}{A_q^{+}}
\newcommand{\Aminus}{A_q^{-}}

\newcommand{\Gall}{G_{q,\tau}^{\forall}}
\newcommand{\Gmix}{G_{q,\tau}^{\mathrm{mix}}}
\newcommand{\Gnone}{G_{q,\tau}^{\varnothing}}
\newcommand{\Gex}{G_{q,\tau}^{\exists}}
\newcommand{\pall}{p_{q,\tau}^{\forall}}
\newcommand{\pmix}{p_{q,\tau}^{\mathrm{mix}}}
\newcommand{\pnone}{p_{q,\tau}^{\varnothing}}
\newcommand{\pex}{p_{q,\tau}^{\exists}}
\newcommand{\Pall}{P_{q,\tau}^{\forall}}
\newcommand{\Pex}{P_{q,\tau}^{\exists}}
\newcommand{\profile}{\Pi_{q,\tau}}

\newcommand{\xicell}{\xi}                      
\newcommand{\Emb}{\Phi}                        
\newcommand{\Jred}{J}                          
\newcommand{\Jstar}{J_{*}}
\newcommand{\Vstar}{V_{*}}

\newcommand{\nucsphere}{\mathbb{S}_{*}}
\newcommand{\Eg}{E_g}
\newcommand{\smin}{s_{-,\tau}}
\newcommand{\smax}{s_{+,\tau}}

\newcommand{\Kinf}{K_{\infty,\tau}}
\newcommand{\vertexV}{\mathsf{v}}              
\newcommand{\edgeE}{\mathsf{e}}                
\newcommand{\Udefect}{U_{q,\tau}}
\newcommand{\Xexc}{\mathcal{X}_{\mathrm{exc}}} 

\newcommand{\ie}{i.e.\@\xspace}


%% file: section/01-introduction.tex
\section{Introduction}\label{sec:intro}

\emph{Grokking} is delayed generalisation: a network fits its training
examples, remains near chance on held-out inputs for a long period, and
then generalises abruptly \citep{power2022grokking}.  The classical
instance is modular addition, where the network observes some pairs
$(x,y)$ with their sums $x+y\bmod q$ and must infer the rest.  Such a
learning curve combines two things, the quality of the predictor that
training eventually reaches and the time taken to reach it, and existing
work examines how learned representations, regularisation, and the
training data shape this behaviour
\citep{nandaProgressMeasuresGrokking2022,mohamadiWhyYouGrok2024,guBeyondProgressMeasures2025}. However, a run that generalises shows that there exsits one solution does; it does not
show that every solution a training could have selected would.
We isolate the generalisation behaviour and ask it at the level of the task. Defining the minimum-norm endpoint as the solution to the weight decay-induced optimization problem, our
central question is therefore: \emph{which task structures make the
entire minimum-norm endpoint set generalise?}

Answering it needs modelling of the task and endpoint set.  A
pairwise input task assigns an output to every pair of inputs, and we encode it
as an \emph{edge colouring} of the complete bipartite graph $K_{q,q}$: the
two input alphabets are the two vertex classes, each pair $(u,v)$ is an
edge, which we also call a cell, and the output is its colour.  Written
out, the colouring is a $q\times q$ table, and at prime-power order it is
a polynomial over $\mathbb F_q$; table, graph, and polynomial are three
coordinates on one object.  The colouring is \emph{Latin} when it is
proper, so that no vertex sees a colour twice; in table coordinates every
symbol appears exactly once in each row and column, so the Latin
colourings are the Latin squares of combinatorics.  Every group
multiplication table is one, and the family therefore contains modular
addition together with its natural competitors.  Permuting the vertices
within either class, permuting the colours, or swapping the two classes
only changes how the table is written; we call such a move a
\emph{relabelling}, and the orbit of a colouring under relabellings is
our task object.  Training observes a uniformly random subset of the
cells of fixed size.

The network reads an edge through its two endpoints, with one embedding
table per vertex class, a hidden ReLU layer, and a head over the colours.
For the endpoint of training we follow implicit-bias theory, which
motivates the minimum-norm interpolator as the idealised late-training
object \citep{soudry2018implicit,lyu2020gradient,neyshabur2015norm}.  A
partially observed colouring, however, has many minimum-norm
interpolators: all meet the margins on the observed cells and are free to
disagree on the rest, and a handful of training runs yields a few
selected predictors without certifying that they belong to this set or
how the set behaves as a whole.  We therefore take the whole set, the
\emph{endpoint set}, as the object of study, and describe it by what
holds across it at a fixed held-out threshold $\tau$: an observed
instance is \textsf{all-good} if every minimum-norm interpolator
generalises, \textsf{no-good} if none does, and \textsf{mixed} if both
kinds coexist.  Averaged over the random observations, the three
probabilities form the task's \emph{phase profile}, a point of a triangle
(\Cref{fig:intro}, middle and right).

\begin{figure}[t]
  \centering
  \vspace{-8pt}
  \includegraphics[width=0.95\linewidth]{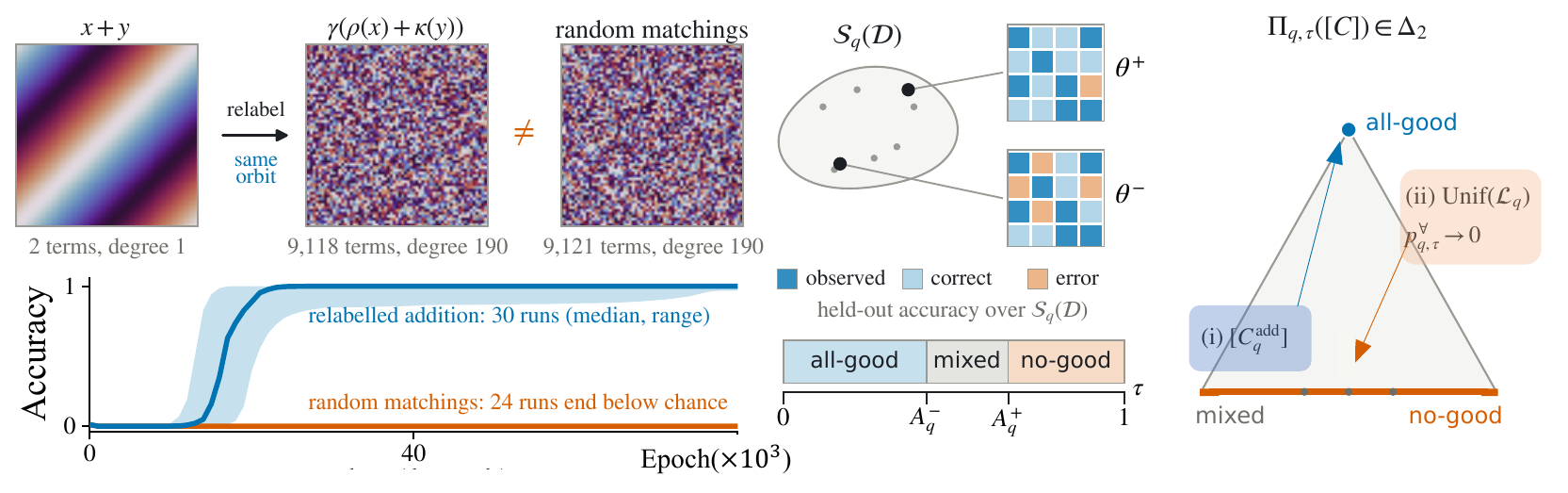}
    \caption{\textbf{Coordinates, endpoint sets, and the two poles.}
      Left: $x+y \bmod 97$; a relabelling of it, $9{,}118$ terms and degree
      $190$; and a random-matching colouring outside the orbit.
      Middle: the minimum-norm interpolators $\mathcal S_q(\mathcal D)$ agree
      on the observed cells, may differ elsewhere; their extremes
      $A_q^{-}\le A_q^{+}$ cut $\tau$ into \textsf{all-good}, \textsf{mixed},
      \textsf{no-good}.  Right: each orbit maps to a simplex point; abelian
      group families sit at the all-good vertex (cyclic addition
      illustrated), while the uniform Latin ensemble sits on the opposite
      edge, where the mixed/no-good split is open.}
  \label{fig:intro}
  \vspace{-8pt}
\end{figure}

Relabelling induces a norm-preserving bijection between endpoint sets and
preserves the law of the observations, so the phase profile is well
defined on orbits: it is a map from colouring orbits to the triangle.
The network realises the bijection by permuting the rows of its embedding
tables and of its head.  This separates task structure from its
polynomial description.  Take $x+y \bmod 97$ and relabel it.  The
endpoint set is transported by the bijection, but the polynomial
coordinate becomes a polynomial over $\mathbb F_{97}$ with $9{,}118$
nonzero terms and total degree $190$, with essentially the same support
density and the same degree as a proper colouring assembled from random
perfect matchings, which lies outside the orbit.  Networks trained on the
relabelled colouring memorise and then generalise as they do on $x+y$;
networks trained on the random-matching colouring end below chance
(\Cref{fig:intro}, left).  Relabelling preserves the phase while changing
polynomial sparsity and degree, and in this comparison those coordinate
statistics also fail to distinguish the trained predictors' outcomes.
This is why we work with orbits.

Our results place two structures at opposite poles of the triangle.  The all-good pole contains \emph{every finite
abelian group family}: at any fixed positive observation fraction, the
worst endpoint accuracy converges to one in probability as the order
grows.  Cyclic addition is one instance, and non-isomorphic abelian
groups supply distinct all-good orbits.  At the opposite pole, a
uniformly random Latin colouring defeats every learner that reads only
the observations, at any fixed positive accuracy; consequently the
colourings with a non-negligible all-good component have vanishing
uniform mass.  Generalisation is thus guaranteed throughout a structured
family yet rare in the ambient space.

The core of our theory is a lemma of \emph{carrier} criterion: a network that scores the
full grid with controlled margins at a known norm cost.  It bounds the
cost of every minimum-norm endpoint on every mask, and a finite atlas
fixed before the observations counts the harmful low-norm alternatives,
whose number is too small for any of them to evade a random mask.  The
criterion does not depend on the colouring, and abelian characters
provide carriers uniformly across group structures.  Therefore, we have proven that satisfying this lemma is a sufficient condition for all good. We conjecture that
the carrier criterion is also necessary: every asymptotically all-good
Latin family admits a full-grid carrier in the same norm-and-margin
regime.

Experiments at composite order $64$ compare the guaranteed families with
non-abelian and non-group candidates and show that coarse output factors
are recovered before exact labels.  These trained predictors probe
structure; they do not measure the phase of the whole endpoint set.
Whether a typical random colouring has some good endpoint, and whether
any family carries non-vanishing mixed mass, remain open.

\noindent\textbf{Contributions.}
\begin{itemize}[leftmargin=1.2em,itemsep=0pt,topsep=0pt,parsep=0pt,partopsep=0pt]
  \item A task-level formulation: Latin colourings of $K_{q,q}$ compared
        up to relabelling, each observed instance judged by its whole
        minimum-norm solution set.
  \item Two poles: every finite abelian group family is asymptotically
        all-good, while the uniform Latin ensemble defeats every learner
        at any fixed positive held-out accuracy.
  \item A colouring-independent carrier criterion for all-goodness,
        together with the conjecture that this criterion is necessary.
  \item Experiments in the same model class: relabelled addition groks like
        addition while looking like noise, and at composite order structured
        tasks recover coarse output factors before exact labels.
\end{itemize}

\textbf{Paper Organization.}
\Cref{sec:related} reviews prior work, \Cref{sec:model} defines the
model, \Cref{sec:overview} states and proves the two poles and the
abelian extension, \Cref{sec:open-region} reports the composite-order
experiments and states the conjecture, and \Cref{sec:conclusion}
concludes. Experiments that speak directly to the theory follow the theory itself.

%% file: section/03-related-work.tex
\section{Related Work}\label{sec:related}
Grokking is reported in non-neural models
\citep{mallinarEmergenceNonneuralModels2024}, deep networks
\citep{fanDeepGrokking2024,manirSystematicEmpiricalGrokking2026}, in-context and
multi-hop reasoning
\citep{heLearningGrokEmergence2024,wangGrokkedTransformersAre2024,abramov2025grokking},
and LLM pretraining \citep{li2025grokkingllm}; its onset needs a critical data
fraction \citep{zhuCriticalDataSize2024,huangUnifiedViewGrokking2024} and its
delay can be shortened
\citep{thilakSlingshotMechanismEmpirical2022a,leeGrokfastAcceleratedGrokking2024,
parkAccelerationGrokkingLearning2024,kuninGetRichQuick2024,xu2025letmegrok}.
Related approaches include representation learning, norm and optimisation
bias, and circuit competition.  \emph{Representation-level} work studies Fourier circuits and progress
measures \citep{nandaProgressMeasuresGrokking2022,gromov2023grokkingmodulararithmetic,furutaInterpretingGrokkedTransformers2024,
barakHiddenProgressDeep2022}, spectral dynamics
\citep{zhouRationaleFrequencyPerspective2024,yunisApproachingDeepLearning2024},
embedding bottlenecks \citep{alquabeh2025mechanistic}, architectural bias
\citep{singh2026explaining,yildirim2026geometric}, task symmetry
\citep{hwang2026intrinsic}, embedding uniformity coupled to the training
distribution \citep{guBeyondProgressMeasures2025}, feature-emergence scaling laws
\citep{tian2025provable}, and topological signatures \citep{tang2026topological}.
\emph{Norm- and optimisation-based} accounts appeal to implicit bias toward
max-margin, minimum-norm endpoints
\citep{soudry2018implicit,lyu2020gradient,neyshabur2015norm,Lyu2024dichotomy,
mohamadiWhyYouGrok2024,humayun2024deep}, to norm minimisation on the zero-loss
manifold \citep{musat2025geometry,notsawo2025grokking}, to the edge of linear
separability \citep{beck2025grokking}, to numerical stability
\citep{prieto2025grokking}, sharpness \citep{millerMeasuringSharpnessGrokking2024},
Jacobian alignment \citep{walker2025grokalign}, and quantitative delay laws
\citep{khanh2026norm,tan2026deciphering,ersoy2026noise}.  \emph{Competition}
accounts pit a slow generalising circuit against a fast memorising one
\citep{varmaExplainingGrokkingCircuit2023,merrill2023grokking}, linking grokking
to double descent and emergent abilities
\citep{davies2023unifying,huangUnifiedViewGrokking2024,liu2022towards}, with
phase-transition, statistical, basin-selection, and capacity variants
\citep{clauw2024information,carvalho2025grokking,cullen2026basin,
song2026model,luo2026stochastic}.
Provable results include generalisation guarantees for bounded-norm
quadratic networks on modular addition \citep{mohamadiWhyYouGrok2024}
and delay bounds in ridge regression \citep{xu2026grokgrokking}.
\citet{guBeyondProgressMeasures2025} separately analyse the accuracy
ceiling and embedding dynamics; \Cref{app:discussion-timescale} discusses
their ceiling and transition-time measurements.
Our setting fixes a ReLU architecture and a uniform mask law, compares
Latin colourings up to relabelling, and quantifies over every minimum-norm
response.  These choices specify the scope of our endpoint guarantees.

%% file: section/02-model.tex
\section{Latin-Square Tasks and the Endpoint Set}
\label{sec:model}

This section builds the objects the introduction used informally: the
task and its observed part, the endpoint set, its three phases, and the
relabelling invariance that makes the phase a property of the task rather
than of its presentation.

\subsection{Tasks, masks, and three coordinates}
\label{sec:model-task}

Let $q\ge2$, $\qset=\{0,\ldots,q-1\}$, $\Xq=\qset^2$, and $N=q^2$.  A
pairwise classification task is a colouring $C:\Xq\to\qset$; the set of all
colourings is $\Cq$.  A colouring is \emph{Latin} when every symbol occurs
exactly once in each row and each column; $\Lq\subset\Cq$ denotes the Latin
colourings of order $q$.  For every $q$, modular addition
$\Cadd(u,v)=u+v \pmod q$ is Latin and is the cyclic prototype of the
abelian all-good family proved below.

Fix $\alpha\in(0,1)$.  The observed cells form a uniformly random mask of
fixed size,
\begin{equation}\label{eq:mask}
  \Omega\sim\nu_q:=\Unif\bigl\{\Omega'\subseteq\Xq:\card{\Omega'}=\lfloor\alpha N\rfloor\bigr\},
  \qquad
  T=\Xq\setminus\Omega ;
\end{equation}
$T$ is the held-out set and
$\Dset(C,\Omega)=\{(u,v,C(u,v)):(u,v)\in\Omega\}$ the dataset.  All
asymptotic statements hold at fixed $\alpha$.

\begin{proposition}[Two coordinates, and a third at prime-power order]\label{thm:bipartite-model}
For every $q\ge2$, a label table $C\in\Cq$ and a $q$-edge-colouring of
$K_{q,q}$ (colour $C(u,v)$ on the edge $U_uV_v$) determine each other.  When
$q$ is a prime power and $\qset$ is identified with $\mathbb F_q$, both are
also equivalent to the unique reduced polynomial $P_C\in\mathbb F_q[X,Y]$ with
$\deg_X P_C,\deg_Y P_C\le q-1$ and $P_C(u,v)=C(u,v)$.  $C$ is Latin if and
only if the edge colouring is proper, \ie each colour class is a perfect
matching.
\end{proposition}

\subsection{The minimum-norm endpoint set}
\label{sec:model-endpoint}

The model is a bias-free one-hidden-layer ReLU network with a separate
embedding table for each input, $E_U,E_V:\qset\to\Rbb^{d_q}$, a hidden
layer $W\in\Rbb^{h_q\times d_q}$, and a head $A\in\Rbb^{q\times h_q}$, where
$d_q=2q+1$ and $h_q=q^2$; the width allows one hidden unit per cell, which
is what feasibility needs, and the embedding dimension exceeds the rank of
any reduced representation.  Writing
$\theta=(E_U,E_V,W,A)$,
\begin{equation}\label{eq:network}
  F_\theta(u,v)=A\,\relu\!\bigl(W\,(E_U(u)+E_V(v))\bigr)\in\Rbb^q,
  \;
  \lVert\theta\rVert^2
  =\frob{E_U}^2+\frob{E_V}^2+\frob{W}^2+\frob{A}^2 .
\end{equation}
As a function class the pair embedding contains the common design that
shares one table $E$ and concatenates $[E(u);E(v)]$, since
$W[E(u);E(v)]=W_1E(u)+W_2E(v)$; the norms differ, so the endpoint sets need
not coincide.  The
\emph{endpoint set} of an observed dataset is the set of global minimisers of
the parameter norm under unit margins on the observed cells,
\begin{equation}\label{eq:minnorm}
  \Sq(\Dset)
  =\argmin_{\theta}\ \tfrac12\lVert\theta\rVert^2
  \quad\text{s.t.}\quad
  F_{\theta,y}(u,v)\ \ge\ 1+\max_{a\ne y}F_{\theta,a}(u,v)
  \qquad\bigl((u,v,y)\in\Dset\bigr).
\end{equation}
Max-margin and implicit-bias results for homogeneous networks motivate this
as the idealised late-training object
\citep{soudry2018implicit,lyu2020gradient,neyshabur2015norm}.  $\Sq(\Dset)$
is a set: all its members meet the margins on $\Omega$, and they may predict
differently on $T$.  \Cref{app:bias} treats the extension to
$L_2$-regularised biases.

Three of the four parameter blocks can be folded into one matrix, which
is the form the proofs work with.  Assemble
$\Emb=[E_U(0),\ldots,E_U(q-1),E_V(0),\ldots,E_V(q-1)]\in\Rbb^{d_q\times2q}$,
put $Z=W\Emb$ and $\xicell_{(u,v)}=e_u+e_{q+v}\in\Rbb^{2q}$, so that
$F_\theta(x)=A\relu(Z\xicell_x)$.  Call $(A,Z)$ a \emph{reduced
representation} and define
\begin{equation}\label{eq:Jred}
\begin{aligned}
  \Jred(F)&=\inf\bigl\{\frob{A}\,\nucnorm{Z}:\ F=A\relu(Z\xicell_{\cdot}),\
  \text{any finite width}\bigr\},\\
  \Jstar&=\min\bigl\{\frob{A}\,\nucnorm{Z}:\ Z\in\Rbb^{h_q\times2q},\
  A\relu(Z\xicell_{\cdot})\text{ meets the margins of \eqref{eq:minnorm}}\bigr\}.
\end{aligned}
\end{equation}

\begin{lemma}[Reduced normalisation]\label{lem:reduced-normalization}
The optimal value of \eqref{eq:minnorm} is $\Vstar=\tfrac32\Jstar^{2/3}$, and
every global minimiser is balanced:
$\frob{A}^2=\frob{W}^2=\frob{\Emb}^2=\Jstar^{2/3}$ and
$\frob A\nucnorm{W\Emb}=\Jstar$.
\end{lemma}

In words: the nuclear norm of $Z$ is the cost of its cheapest factorisation
into a hidden layer and an embedding, and the scale freedom of the ReLU then
balances the head against the rest, so an endpoint is simply a score table
whose product $\frob A\nucnorm Z$ is as small as the observations allow.
This product is the one quantity the main proof tracks.

\subsection{Phases of the endpoint set and the phase profile}
\label{sec:model-phases}

Fix $C$ and $\Omega$ and write $\Sq=\Sq(\Dset(C,\Omega))$.  Strict held-out
accuracy counts a score tie as an error, and the best and worst endpoint
accuracies are its extrema over the endpoint set:
\begin{equation}\label{eq:acc}
  \Acc_T(\theta;C)=\frac1{\card T}\sum_{x\in T}
  \mathbf 1\Bigl\{F_{\theta,C(x)}(x)>\max_{a\ne C(x)}F_{\theta,a}(x)\Bigr\},
  \;
  \Aplus=\max_{\Sq}\Acc_T,\ \
  \Aminus=\min_{\Sq}\Acc_T .
\end{equation}
For a threshold $\tau\in(0,1)$ the endpoint set is in exactly one of three
phases, each a quantifier over $\Sq$:
\begin{equation}\label{eq:phase-quantifiers}
  \textsf{all-good}:\ \Aminus\ge\tau,\qquad
  \textsf{mixed}:\ \Aplus\ge\tau>\Aminus,\qquad
  \textsf{no-good}:\ \Aplus<\tau .
\end{equation}
Write $\Gall,\Gmix,\Gnone$ for their indicators and $\Gex=\Gall+\Gmix$ for
``some endpoint is good''.  Averaging over the mask, for
$\star\in\{\forall,\mathrm{mix},\varnothing,\exists\}$, gives the phase
probabilities and the \emph{phase profile},
\begin{equation}\label{eq:phase-probabilities}
  p^\star_{q,\tau}(C)=\Exp_{\Omega\sim\nu_q}\bigl[G^\star_{q,\tau}(C,\Omega)\bigr],
  \qquad
  \profile(C)=\bigl(\pall(C),\,\pmix(C),\,\pnone(C)\bigr).
\end{equation}

\begin{proposition}[Phase decomposition]\label{thm:phase-profile}
$\Sq(\Dset)$ is non-empty for every $C\in\Cq$ and $\Omega\subseteq\Xq$.  For
every $\tau\in(0,1)$ the three phases in \eqref{eq:phase-quantifiers} are
mutually exclusive and exhaustive, so $\profile(C)$ lies in the simplex
$\simplex=\{(a,b,c)\in[0,1]^3:a+b+c=1\}$.
\end{proposition}

The threshold only decides where an instance falls; the extremes
$\Aminus\le\Aplus$ are properties of the set itself.  A fixed instance is
all-good for thresholds up to $\Aminus$, mixed between $\Aminus$ and
$\Aplus$, and no-good above, so it is mixed at \emph{some} threshold exactly
when different endpoints achieve different accuracies
(\Cref{app:finite-criterion}).

\subsection{Invariance: the profile lives on colouring orbits}
\label{sec:model-invariance}

Permuting rows, columns, or symbols, or exchanging the two inputs, changes
how a table is written and nothing about the learning problem.  The formal
name for this equivalence is colouring isomorphism (for Latin squares,
isotopy together with transposition); the paper calls it relabelling.

\begin{definition}[Colouring isomorphism]\label{def:quotient}
$C\sim_q C'$ when there are $\rho,\kappa,\gamma\in S_q$ with
$C'(u,v)=\gamma(C(\rho(u),\kappa(v)))$ for all $u,v$, or
$C'(u,v)=\gamma(C(\kappa(v),\rho(u)))$ for all $u,v$.  The quotients are
$\Qallq=\Cq/\!\sim_q$ and $\Qq=\Lq/\!\sim_q$.
\end{definition}

\begin{theorem}[Quotient invariance]\label{thm:quotient-invariance}
Let $C\sim_q C'$.  There are bijections of fixed-size masks and of endpoint
sets, preserving the parameter norm and the mask law, under which
corresponding endpoints have the same logits, the same truth--rival margins,
and the same strict correctness at every cell.  Consequently
$\pall,\pmix,\pnone,\pex$ and $\profile$ are well defined on $\Qallq$, hence
on $\Qq$, for every $\tau\in(0,1)$.
\end{theorem}

The proof is short: relabelling the task amounts to permuting the entries
of the two embedding tables and the rows of the head, and swapping the two
tables under transposition, and each of these moves preserves the norm, the
margin constraints, and cellwise correctness.  The consequence is that the
profile becomes a function on orbits, the \emph{classification map}
\begin{equation}\label{eq:informal-classification-map}
  \Qq\longrightarrow\simplex,\qquad [C]\longmapsto\profile([C]),
\end{equation}
which is the object classified in the rest of the paper.  One consequence
deserves to be seen before the full proof (\Cref{app:struct-proofs}):
the support size and degree of the polynomial coordinate are not orbit
invariants, so every relabelled representative $\gamma(\rho(u)+\kappa(v))$
of the addition orbit has the profile of $\Cadd$, however dense its
polynomial.

\begin{wrapfigure}{r}{0.6\linewidth}
  \vspace{-11pt}
  \centering
  \includegraphics[width=\linewidth]{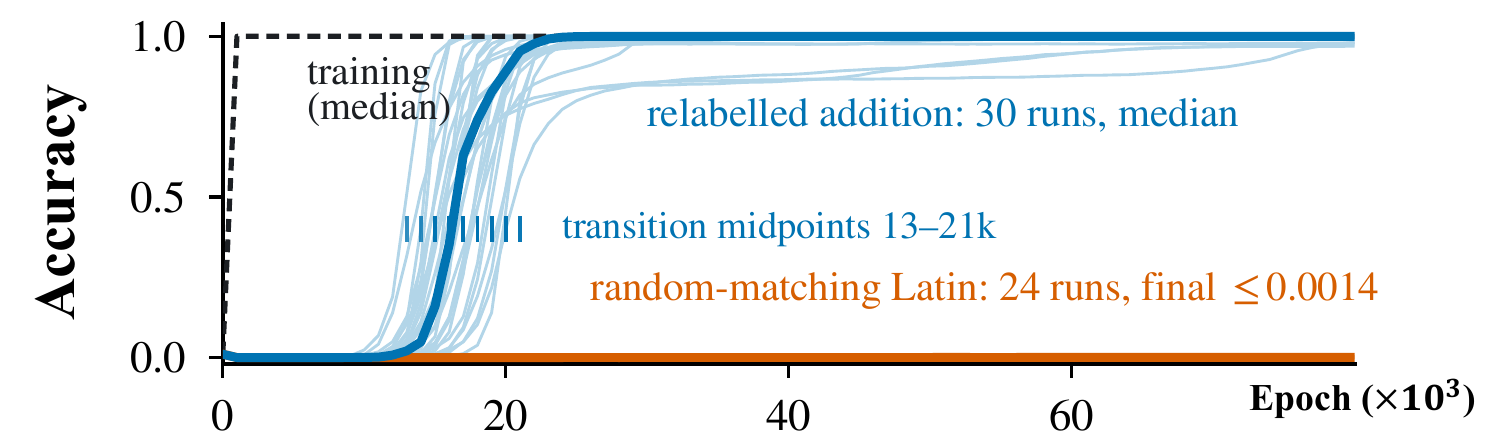}
  \vspace{-10pt}
  \caption{\textbf{Relabelled addition groks like addition.}}
  \label{fig:orbit-representation}
  \vspace{-5pt}
\end{wrapfigure}
\textbf{Dense representatives of addition grok.}
\Cref{fig:orbit-representation} is the invariance theorem seen from the
data side.  Ten random relabellings $(\rho,\kappa,\gamma)$ of $\Cadd$ at
$q=97$ give tables whose reduced polynomials have total degree $190$ and
$97\%$ coefficient support, matching the random-matching Latin controls in
both statistics, while their orbit is still that of $\Cadd$.  All $30$ runs memorise and then generalise, ending at
median held-out accuracy $1.000$; the $24$ controls fit their training
cells and end below $0.0014$ held-out.  Reusing a symbol already seen in
the row or column of a cell is always wrong by the Latin property, so
below-chance accuracy is possible.  Within the
orbit final accuracies cluster near one but the delays vary with the optimisation
seed: a random-effects decomposition of the transition midpoint attributes
about $93\%$ of its variance to it (\Cref{app:orbit-details}).  At this
order, these runs support a stable final outcome with variable delay;
they do not certify membership in $\Sq$ or measure its phase.  The next
section states the endpoint guarantees.

%% file: section/03-overview.tex
\section{Two Poles of the Latin Quotient}
\label{sec:overview}
\textbf{Informal Results.}
The objects of the previous section turn the question of the introduction
into a map: each relabelling orbit of Latin colourings is sent to its phase
profile, a point of the triangle whose corners are all-good, mixed, and
no-good.  Which orbits sit at the all-good corner?  Which carry mixed
mass?  And where does the orbit of a uniformly random Latin colouring sit?  Our results answer the
first for every abelian group family and the third in its all-good coordinate, with the
observation fraction and the success threshold held fixed while the order
grows.
\begin{theorem*}[Two poles of the Latin quotient]
Fix an observation fraction and a success threshold and let the order tend
to infinity.
\begin{enumerate}[label=(\roman*),leftmargin=1.8em,itemsep=0pt,topsep=0pt,parsep=0pt,partopsep=0pt]
  \item \emph{Abelian group families are all-good:} along any sequence
        of finite abelian groups of growing order, the worst minimum-norm
        endpoint accuracy converges to one in probability, so the profile
        converges to the all-good corner.
  \item \emph{Uniform Latin squares are unlearnable from the observations:}
        for a uniformly random Latin colouring, no learner whose prediction
        depends only on the observed cells reaches the success threshold
        with probability bounded away from zero.  Consequently orbits with
        a non-negligible all-good component have vanishing mass under this law.
\end{enumerate}
\end{theorem*}
Part~(i) concerns every endpoint of each task in a structured family;
part~(ii) concerns every learner under a random task law.  The latter
controls the all-good coordinate, leaving open whether a typical
colouring has some good endpoint.

%% file: section/04-main-theorem.tex
\subsection{The main results}
\label{sec:main}

We prove cyclic addition first and obtain every abelian group from it as
a corollary, rather than proving the general statement directly: both
proofs share one method, a carrier and a notch; the notch is insensitive
to the colouring, and only the characters that build the carrier change
from $\Zbb_q$ to a general abelian group.  The cyclic case thus shows the
whole argument in its most concrete form.  Part~(ii) needs two pieces
of vocabulary: a task ensemble, since it concerns a random Latin square,
and a learner, since it quantifies over everything that reads the
observations.  A \emph{task ensemble} is a sequence $\mathsf M=(\mu_q)_q$ of distributions
on $\Lq$; we use $\Madd=(\delta_{\Cadd})_q$ and $\Munif=(\Unif(\Lq))_q$.  A
\emph{learner} $\mathsf A_q$ is a deterministic map from observed datasets
to score tables on $\Xq$, that is, to functions $\Xq\to\Rbb^q$ read through
$\Acc_T$ as in \eqref{eq:acc} (a randomised learner is covered by
conditioning on its coins).  With $C\sim\mu_q$ and $\Omega\sim\nu_q$ independent,
\begin{equation}\label{eq:selector-probability}
  P^\star_{q,\tau}(\mathsf M)=\Exp_{C\sim\mu_q}\bigl[p^\star_{q,\tau}(C)\bigr],
  \qquad
  P^{\mathsf A}_{q,\tau}(\mathsf M)
  =\Prob\bigl[\Acc_T(\mathsf A_q(\Dset(C,\Omega));C)\ge\tau\bigr].
\end{equation}
A \emph{selector} is a learner with $\mathsf A_q(\Dset)=F_\theta$ for some
$\theta\in\Sq(\Dset)$; since $\Sq(\Dset)$ depends on the observations alone,
any choice function on it is a learner, and one exists because the
endpoint set is never empty.

\begin{theorem}[Cyclic prototype and uniform obstruction]\label{thm:main}
Fix $\alpha,\tau\in(0,1)$ and let $q\to\infty$.
\begin{enumerate}[label=(\roman*),leftmargin=2.2em,itemsep=2pt,topsep=2pt]
  \item\label{it:main-add} \textup{(Addition is all-good.)}  The worst
        endpoint accuracy converges to one in probability over the mask,
        \begin{equation}\label{eq:addition-allgood}
          \Aminus(\Cadd,\Omega)\xrightarrow{\ \mathrm P\ }1,
          \quad\text{hence}\quad
          \Pall(\Madd)\to1
          \quad\text{and}\quad
          \profile([\Cadd])\to(1,0,0).
        \end{equation}
  \item\label{it:main-unif} \textup{(Uniform Latin squares defeat every learner.)}
        Uniformly over all learners $\mathsf A_q$,
        \begin{equation}\label{eq:separation}
          P^{\mathsf A}_{q,\tau}(\Munif)\ \le\ (2e^2)^{N}\,q^{-\tau(1-\alpha)N}
          \ \longrightarrow\ 0,
        \end{equation}
        hence $\Pall(\Munif)\to0$ and
        $\Prob_{C\sim\Unif(\Lq)}\bigl[\pall(C)\ge\varepsilon\bigr]\to0$ for
        every $\varepsilon>0$.
\end{enumerate}
\end{theorem}

The two parts are proved in the next two subsections.  The second part is
pure counting and never looks at the model; the phase enters only through
the observation that on an all-good instance every selector succeeds, so
$\Pall(\Munif)\le P^{\mathsf A}_{q,\tau}(\Munif)$, and Markov's inequality
turns the vanishing expectation into the vanishing-mass statement.  Nothing
in either proof depends on the arithmetic of $q$: the theorem holds along
every sequence of orders, with $\Cadd$ the cyclic group table of order $q$.
Written as $\exp\bigl(-N[\tau(1-\alpha)\log q-\log 2e^2]\bigr)$, the
bound in \eqref{eq:separation} becomes informative once
$q^{\tau(1-\alpha)}>2e^2$; the statement about orbits is the push-forward
of the uniform law on squares.

\subsubsection{Addition: a carrier and a notch}
\label{sec:main-addition}

The all-good pole rests on two lemmas.  The first exhibits a cheap network
that computes the whole addition table with a comfortable margin; the second
says that, with high probability over the observed cells, no cheap network
can agree with the observations and still be confidently wrong on a constant
fraction of the table.

\begin{lemma}[Addition carrier]\label{lem:addition-carrier}
There is an absolute constant $C_0$ such that for every sufficiently large
$q$ there is a full-grid score table
$\Ccar=A_{\rm car}\relu(Z_{\rm car}\xicell_{\cdot})$ of width $w_q=O(\log q)$
whose truth--rival gaps and product norm satisfy, for all $x\in\Xq$ and
$a\ne\Cadd(x)$,
\begin{equation}\label{eq:carrier-margins}
  1\ \le\ \Ccarc{\Cadd(x)}(x)-\Ccarc{a}(x)\ \le\ 9,
  \qquad
  J_{\rm car}:=\frob{A_{\rm car}}\,\nucnorm{Z_{\rm car}}\ \le\ C_0\,q\sqrt{\log q}.
\end{equation}
\end{lemma}

The carrier is the Fourier-multiplication circuit of the interpretability
literature, made explicit and costed.  Roughly $\log q$ frequencies suffice
for their averaged cosine to separate the correct sum from every wrong one by
a constant gap; a small ReLU expansion of a piecewise-linear square approximates
the required products of cosines and sines, and each of its hinges
splits into a row part plus a column part, which is what the additive
architecture needs.  A constant number of hidden units per frequency gives
product norm $O(q\sqrt{\log q})$ (\Cref{app:carrier-proof}).

\begin{lemma}[One-sided notch]\label{lem:one-sided-notch}
Fix $\alpha,\beta\in(0,1)$ and let $B_q$ be a deterministic sequence with
$B_q^2\log q/q^{5/2}\to0$.  With probability $1-o(1)$ over $\Omega\sim\nu_q$
there is no table $D(x)=A\relu(Z\xicell_x)$, of any finite width, with
$\frob{A}\nucnorm{Z}\le B_q$ that satisfies both
\begin{enumerate}[label=(\arabic*),leftmargin=2.2em,itemsep=1pt,topsep=1pt]
  \item $D_{y(x)}(x)-D_a(x)\ge0$ for every $x\in\Omega$ and $a\ne y(x)$,
        where $y(x)=\Cadd(x)$; and
  \item at least $\beta N$ cells $x$ admit a rival $a(x)\ne y(x)$ with
        $D_{a(x)}(x)-D_{y(x)}(x)\ge1$.
\end{enumerate}
\end{lemma}

Both names are descriptive: the carrier carries the whole table at a known
cost, and the notch is the region of the norm--error plane the lemma cuts out,
cheap tables that fit the observations yet err confidently on a constant
fraction of cells.  The difficulty is that $D$ may be chosen \emph{after} the
mask is seen, and the proof meets it in three steps.  Before the mask is drawn
it fixes a finite atlas $\mathcal H_q$ of suspect cell sets, such that the
wrong cells of every cheap confidently-wrong table are covered, up to a
vanishing fraction of $N$, by some member.  The atlas is small: splitting $Z$
at rank about $\sqrt q$, replacing both parts by a few sampled atoms, covering
the sampled core by integer codes
\citep{zhang2002covering,ledent2021norm}, and counting the sign patterns of
the resulting polynomial gates \citep{caro2020pseudo} give
$\card{\mathcal H_q}=\exp(o(N))$, a count that needs exactly the scaling
assumed of $B_q$.  And a fixed set of $\beta N/2$ cells meets a uniform mask in
too few cells only with probability $e^{-cN}$, while a cheap table that is
confidently wrong on an observed cell violates its margin there
(\Cref{app:notch-proof}).  Part~\ref{it:main-add} is therefore the competition
\begin{equation}\label{eq:competition}
  \Prob\bigl(\exists\ \text{harmful low-norm table}\bigr)
  \ \le\ \card{\mathcal H_q}\,e^{-cN}
  \ =\ \exp\bigl(\log\card{\mathcal H_q}-cN\bigr)\longrightarrow0,
\end{equation}
the carrier serving only to put every endpoint inside the budget $B_q$; width
enters through the budget alone (\Cref{app:discussion-atlas}).

\begin{proof}[Proof of \Cref{thm:main}\ref{it:main-add}]
Padded with zero hidden rows to width $N$, $(A_{\rm car},Z_{\rm car})$ is
feasible for the minimisation defining $\Jstar$ on every addition mask, so
$\Jstar\le J_{\rm car}\le C_0q\sqrt{\log q}$, and by
\Cref{lem:reduced-normalization} each endpoint response $F$ has a width-$N$
reduced representation $(A_F,Z_F)$ with $\frob{A_F}\nucnorm{Z_F}=\Jstar$.
Take $D=9F-\Ccar$ and $y(x)=\Cadd(x)$.  On an observed cell $F_y-F_a\ge1$
against a carrier gap of at most $9$, so $D_y-D_a\ge0$; at a held-out
non-strict error, including a tie, some rival has $F_a\ge F_y$ against a
carrier gap of at least $1$, so $D_a-D_y\ge1$.  Balancing both representations
by homogeneity and stacking them,
$\widetilde A=[\,3A_F\ \ -A_{\rm car}\,]$ and
$\widetilde Z=[\,3Z_F;\ Z_{\rm car}\,]$ represent $D$ with
$\frob{\widetilde A}^2=9\Jstar+J_{\rm car}$ and
$\nucnorm{\widetilde Z}\le3\sqrt{\Jstar}+\sqrt{J_{\rm car}}\le
\sqrt2\sqrt{9\Jstar+J_{\rm car}}$, so
\begin{equation}\label{eq:JD}
  \Jred(D)\le\frob{\widetilde A}\nucnorm{\widetilde Z}
  \le\sqrt2\,(9\Jstar+J_{\rm car})
  \le10\sqrt2\,C_0\,q\sqrt{\log q}=:B_q .
\end{equation}
Then $B_q^2\log q/q^{5/2}=O((\log q)^2/\sqrt q)\to0$, and the stack has width
$N+w_q$, which \Cref{lem:one-sided-notch} permits.  An endpoint with
$\Acc_T<\tau$ has more than $(1-\tau)\card T\ge(1-\tau)(1-\alpha)N$ non-strict
errors, so with $\beta=\tfrac12(1-\tau)(1-\alpha)$ its table $D$ is excluded by
the lemma.  The exclusion event is a property of the mask, so it holds
simultaneously for all endpoints: with probability $1-o(1)$,
$\Aminus(\Cadd,\Omega)\ge\tau$.  This holds for every fixed $\tau<1$, which is
$\Aminus\xrightarrow{\mathrm P}1$; the phase decomposition gives the profile
limit.
\end{proof}

\begin{corollary}[Abelian group families are all-good]
\label{cor:abelian-allgood}
Let $(G_q)_q$ be any sequence of finite abelian groups with $|G_q|=q\to\infty$,
identify their elements with $\qset$, and let $C_{G_q}(u,v)=u+v$ be their
group tables.  For every fixed $\alpha,\tau\in(0,1)$,
\begin{equation}\label{eq:abelian-allgood}
  \Aminus(C_{G_q},\Omega)\xrightarrow{\mathrm P}1,
  \qquad \profile([C_{G_q}])\to(1,0,0).
\end{equation}
The conclusion holds for every relabelled representative.
\end{corollary}

Cyclic frequencies are the characters of $\Zbb_q$; replacing them by the
characters of $G_q$ preserves the carrier's width $O(\log q)$, gaps
$[1,9]$, and cost $O(q\sqrt{\log q})$, and the notch uses no group
arithmetic, so the same budget proves the corollary.  The general
criterion and full character calculation are in
\Cref{app:carrier-criterion,app:abelian-carriers}.

\subsubsection{Uniform Latin squares: a counting bound}
\label{sec:main-uniform}

\begin{lemma}[Counting]\label{lem:uniform-selector-count}
Let $n_q=\lfloor\alpha N\rfloor$ and $r=\lfloor(1-\tau)(N-n_q)\rfloor$.  For
every learner $\mathsf A_q$,
\begin{equation}\label{eq:selector-count}
  \text{(a)}\quad
  P^{\mathsf A}_{q,\tau}(\Munif)
  \ \le\ \frac{q^{n_q}\sum_{i=0}^{r}\binom{N-n_q}{i}(q-1)^i}{\card\Lq},
  \qquad
  \text{(b)}\quad
  \card\Lq\ \ge\ \frac{(q!)^{2q}}{q^{q^2}}\ \ge\ \Bigl(\frac{q}{e^2}\Bigr)^{q^2}.
\end{equation}
\end{lemma}

The uniform pole needs no geometry.  Part~(a) fixes the mask, notes that
the observed labels can take at most $q^{n_q}$ values, and for each value counts
the tables within Hamming distance $r$ of what the learner predicts;
part~(b) is a lower bound on the number of Latin squares, obtained by
extending a square row by row and applying the van der Waerden permanent
bound to the admissible entries of each new row
\citep{egorychev1981solution,falikman1981proof,gurvits2005hyperbolic}; the
bound is essentially sharp, since
$\card\Lq=\bigl((1+o(1))\,q/e^2\bigr)^{q^2}$ \citep{vanlint2001course}.
The ratio of the two counts is the theorem.

\begin{proof}[Proof of \Cref{thm:main}\ref{it:main-unif}]
Since $\sum_{i\le r}\binom{N-n_q}{i}(q-1)^i\le2^{N}q^{r}$ and
$N-n_q-r\ge\tau(N-n_q)\ge\tau(1-\alpha)N$, the two parts of
\eqref{eq:selector-count} give
$P^{\mathsf A}_{q,\tau}(\Munif)\le 2^Ne^{2N}q^{-\tau(1-\alpha)N}$,
uniformly in $\mathsf A_q$.  On an all-good instance every deterministic
selector succeeds, so $\Pall(\Munif)\le P^{\mathsf A}_{q,\tau}(\Munif)$ for
any such selector.
\end{proof}

\begin{remark}[Biases]
With hidden bias $b$ and output bias $c$ added to \eqref{eq:network} and to
the penalised norm, both assertions of the theorem survive
(\Cref{cor:l2-bias-transfer}): a biased endpoint response has a bias-free
analysis representation of product $O(q^{7/6}(\log q)^{1/3})$, so
$B_q^2\log q/q^{5/2}\to0$ and \eqref{eq:competition} is unchanged.
\end{remark}
\begin{remark}[Vanishing observation fractions]
\label{rem:vanishing-observation}
Applying part~\ref{it:main-add} at $\alpha=1/k$ and choosing orders beyond
which the failure probability at accuracy $1-1/k$ is below $1/k$ yields
$\alpha_q\downarrow0$ along which $\Aminus(\Cadd,\Omega)\xrightarrow{\mathrm P}1$
still holds.  This is existence without a rate; the rate the atlas
calculation suggests, $\alpha_q\gg(\log q)^3/\sqrt q$, is recorded as a
proof-guided scaling question in \Cref{app:proof-guided-scaling}, and the empirical
thresholds of \Cref{tab:addition-coverage} fall from $0.33$ at $q=47$ to
$0.20$ at $q=211$.
\end{remark}
\begin{remark}[What part~\ref{it:main-unif} leaves open]\label{main-remarks}
The existential probability
$\Pex(\Munif)=\Prob[\exists\theta\in\Sq(\Dset(C,\Omega)):\Acc_T(\theta;C)\ge\tau]$
under $C\sim\Unif(\Lq)$, the mixed/no-good boundary of the triangle,
has no known limit, nor is any sequence with non-vanishing mixed mass
known; \Cref{app:classification-readiness} packages the programme.
The non-abelian and non-group extensions of the carrier criterion also
remain open.
    
\end{remark}

%% file: section/05-open-region.tex
\section{Beyond the Poles: Structured Families at Composite Order}
\label{sec:open-region}

At $q=64=2^6$ we compare the abelian all-good families of the previous
section with non-abelian and non-group candidates,
using product factors unavailable at prime order to measure partial label
recovery.  Runs follow \Cref{app:protocol} with label-stratified masks
(distinct from $\nu_q$); they do not measure the phase of $\Sq$.
$C_r$ is the cyclic table, $C_2^k$ the elementary abelian
table, $R_r$ a task-seeded Latin table built from random perfect matchings
(\Cref{app:orbit-details}), and products are direct products.

\begin{wraptable}{r}{0.5\textwidth}
  \vspace{0pt}
  \centering
  \footnotesize
  \setlength{\tabcolsep}{3.5pt}
  \renewcommand{\arraystretch}{1.05}
  \captionsetup{font=small,labelfont=bf,skip=3pt}
  \caption{Final trained predictors at $q{=}64$, $30\%$ coverage, six runs
  per family. ``Transition'' is the plateau-relative event of
  \Cref{app:protocol}; $0.99$ is a reporting cutoff unrelated to $\tau$.}
  \label{tab:group-family-endpoints}
  \begin{tabular}{@{}lccc@{}}
    \toprule
    & \multicolumn{3}{c}{\small final outcome over $6$ runs} \\
    \cmidrule(l{2pt}r{2pt}){2-4}
    Latin family & trans. & med.\ acc.\ & ${\ge}\,0.99$ \\
    \midrule
    $C_2^6$          & $6/6$ & $1.000$ & $6/6$ \\
    $C_{64}$         & $6/6$ & $0.949$ & $3/6$ \\
    $C_8{\times}C_8$ & $6/6$ & $0.941$ & $2/6$ \\
    $\mathrm{Dih}_{64}$ & $6/6$ & $0.870$ & $1/6$ \\
    $C_8{\times}R_8$ & $4/6$ & $0.231$ & $1/6$ \\
    \bottomrule
  \end{tabular}
  \vspace{-4pt}
\end{wraptable}
\noindent\textbf{Group laws beyond cyclic addition.}
Every group-table run in \Cref{tab:group-family-endpoints} transitions;
$C_2^6$ completes exactly, and the other medians are $0.87$--$0.95$.
These finite-training outcomes need not match the asymptotic phase:
$30\%$ is at addition's training threshold at nearby orders
(\Cref{tab:addition-coverage}).  $C_8\times R_8$ yields one run above
$0.99$.  The dihedral and non-group cases motivate further carriers
(\Cref{app:colouring-carriers}).

\begin{figure}[t]
  \centering
  \vspace{-20pt}
  \includegraphics[width=\linewidth]{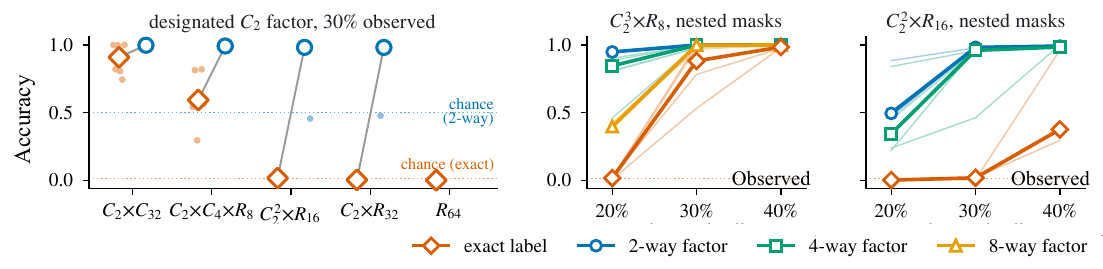}
  \caption{\textbf{Coarse factors are recovered before exact labels.}
  Colour encodes label resolution: exact label and the prespecified $2$-,
  $4$-, $8$-way factor quotients. Left: held-out accuracy at nominal
  $30\%$ coverage (runs small, medians outlined). Right: nested masks at
  $20\%$, $30\%$, $40\%$; coarse factors saturate one coverage level
  before the exact label.}
  \vspace{-8pt}
  \label{fig:open-region}
\end{figure}
\noindent\textbf{Factor recovery with weak exact completion.}
The table $C_2\times R_{32}$ carries a binary label component even when
its $32$-way residual is hard to infer.  Runs are scored on exact labels
and on \emph{factor quotients}: projections of predicted and true labels
to product factors, distinct from the orbit quotient $\Qq$.
In \Cref{fig:open-region}, left, replacing the cyclic residual $C_{32}$
by the random $R_{32}$ lowers median exact recovery to chance or below,
while binary accuracy remains above $0.98$: one label component is
recovered while finer components remain unresolved.

\noindent\textbf{Coverage moves the recovered resolution.}
The nested-mask panels re-train the same task and optimisation seeds at
increasing coverage.  Coarse factors are recovered before exact labels,
with the residual structure setting the required coverage: for
$C_2^3\times R_8$ the exact median rises
$0.017\to0.882\to0.985$ while the binary factor is $0.948$ already at
$20\%$.  This suggests a coordinate for the open region: the finest factor
resolved at a given coverage (\Cref{app:experiment-tables}).
Checkpoint trajectories, alternate factor directions, and probe
interventions (\Cref{app:mechanism-experiments}) confirm that the factor
accuracies reflect structured label resolution, not exact hits alone.

These observations motivate carriers beyond abelian groups and for coarse
factors.  For the all-good limit, we propose the following robust
necessity conjecture.

\begin{conjecture}[Robust margin necessity]
\label{conj:carrier-necessity}
Let $C_q\in\Lq$ be deterministic with $q\to\infty$.  If
$\Aminus(C_q,\Omega)\xrightarrow{\mathrm P}1$ for every fixed
$\alpha\in(0,1)$ under the corresponding uniform mask law, then $C_q$
admits mask-independent carriers of width at most $N=q^2$, with
truth--rival gaps at least $1$ outside $o(N)$ cells and product norm
$\frob{A_{{\rm car},q}}\nucnorm{Z_{{\rm car},q}}=o(q^{3/2})$.
\end{conjecture}
\Cref{app:carrier-necessity} gives the precise formulation, motivation,
and remaining proof gaps.

%% file: section/07-conclusion.tex
\section{Conclusion}\label{sec:conclusion}
We study which task structures make the entire minimum-norm endpoint
set generalise.  Our carrier criterion guarantees this asymptotically
for every finite abelian group family, whereas the uniform Latin
ensemble defeats observation-based learners at fixed positive accuracy.
Moreover, relabelling preserves endpoint generalisation despite changing
polynomial complexity, making task structure central to this
classification.  Composite-order experiments further suggest recovery
through coarse factors, without characterising the whole endpoint set.
Together, these results connect grokking to the generalisation properties
of its idealised endpoints, leaving carrier necessity and the
mixed/no-good boundary as open questions.

\subsection*{AI Use Statement}

Generative AI tools were used to improve the language and clarity of the
manuscript, and to surface existing results in the literature that
shortened the proofs.  The authors located the original sources and
verified each result in the form used; the external proof ingredients
are stated explicitly in \Cref{app:external}.

%% file: main-bib.tex

\bibliographystyle{iclr2027_conference}
\bibliography{references}

%% file: appendix/00-notation.tex
\section{Notation}
\label{app:notation}

\Cref{tab:notation} lists the symbols used throughout the main text.
Appendix-local symbols are defined where they enter.

\begin{table}[h]
  \centering
  \small
  \caption{Notation used in the main text.  Appendix-local symbols are defined
  where they enter.}
  \label{tab:notation}
  \begin{tabularx}{\linewidth}{@{}>{\raggedright\arraybackslash}p{0.24\linewidth}X@{}}
    \toprule
    Symbol & Meaning \\
    \midrule
    $q,\qset,S_q,\Xq,N$ & Task order, symbol set $\{0,\ldots,q-1\}$, its symmetric group, cell grid $\qset^2$, and $q^2$. \\
    $C,\Cq,\Lq,\Cadd$ & A colouring, all pairwise colourings, Latin colourings, and modular addition. \\
    $G_q,C_{G_q}$ & A finite abelian group of order $q$ and its group table, after identifying its elements with $\qset$. \\
    $\Omega,T,\nu_q,\alpha,n_q$ & Observed cells, held-out cells, mask law, coverage, and mask size. \\
    $\Dset(C,\Omega)$ & Observed triples $(u,v,C(u,v))$. \\
    $F_\theta,\theta,\Sq(\Dset)$ & Score map, network parameters, and minimum-norm endpoint set. \\
    $\theta^{\mathrm b},\mathcal S_q^{\mathrm b},\Pi_{q,\tau}^{\mathrm b}$ & Their $L_2$-regularised hidden/output-bias counterparts in \Cref{app:bias}. \\
    $\Acc_T,\tau,\Aplus,\Aminus$ & Strict held-out accuracy, success threshold, and endpoint extrema. \\
    $\Gall,\Gmix,\Gnone,\Gex$ & All-good, mixed, no-good, and existential indicators. \\
    $p_{q,\tau}^\star(C),\profile(C)$ & Threshold-indexed mask-averaged phase probability and phase profile. \\
    $\sim_q,\Qallq,\Qq$ & Colouring isomorphism, ambient quotient space, and Latin quotient space. \\
    $\xicell,\Emb,Z$ & Cell indicator, stacked embedding matrix, and reduced hidden matrix $W\Emb$. \\
    $\Jred,\Jstar,\Vstar$ & Reduced product functional, its constrained optimum, and the original norm optimum. \\
    $\Ccar,B_q$ & Full-grid carrier score table and the product bound in the notch lemma. \\
    $\Madd,\Munif,P_{q,\tau}^\star(\mathsf M)$ & Addition and uniform-Latin task ensembles and their threshold-indexed probabilities. \\
    $P_{q,\tau}^{\mathsf A}(\mathsf M)$ & Success probability of an observation-measurable learner $\mathsf A_q$; a selector is a learner with values in $\Sq$. \\
    $C_r,C_2^k,R_r$ & Cyclic, elementary-abelian, and task-seeded sequential-Latin families used in \Cref{sec:open-region}. \\
    \bottomrule
  \end{tabularx}
\end{table}
\FloatBarrier

%% file: appendix/00-supplementary-theory.tex
\section{Supplementary Theory}
\label{app:supp-theory}

This appendix contains the exact finite-instance criterion behind the phases,
the bias extension of \Cref{thm:main}, an unstructured comparison bound, and a
formal statement of the classification programme that the two poles leave
open.  Proofs are in \Cref{app:proofs}.

\subsection{An exact criterion for a finite instance}
\label{app:finite-criterion}

Fix $(q,C,\Omega)$ and $\tau\in(0,1)$, and allow hidden width exactly $N$
with zero padding.  For $B\in\Rbb^{N\times2q}$ let $H(B)$ have $x$-th column
$\relu(B\xicell_x)$.  On the nuclear sphere $\nucsphere=\{B:\nucnorm B=1\}$
define
\begin{equation}\label{eq:rhoB}
  \rho(B)=\min_A\ \frob A
  \quad\text{s.t.}\quad
  (e_{C(x)}-e_a)^{\top}AH(B)_x\ge1\quad(x\in\Omega,\ a\ne C(x)),
\end{equation}
with $\rho(B)=+\infty$ if no $A$ is feasible, and put
$g=\min_{\nucsphere}\rho$ and $\Eg=\{B\in\nucsphere:\rho(B)=g\}$.  For
$B\in\Eg$ let $A_B$ be the unique minimum-Frobenius solution of
\eqref{eq:rhoB}, $F_B=A_BH(B)$, $m_x(B)=F_{B,C(x)}(x)-\max_{a\ne C(x)}F_{B,a}(x)$
for $x\in T$, $s(B)$ the $k$-th largest held-out margin with
$k=\lceil\tau\card T\rceil$, and
\begin{equation}\label{eq:s-extrema}
  \smin=\min_{B\in\Eg}s(B),\qquad\smax=\max_{B\in\Eg}s(B).
\end{equation}

\begin{theorem}[Exact finite-instance criterion]\label{thm:finite-phase-criterion}
$\Eg$ is non-empty and compact, $B\mapsto F_B$ is continuous on $\Eg$, and the
global responses of \eqref{eq:minnorm} are exactly $\{F_B:B\in\Eg\}$.
Moreover
\begin{equation}\label{eq:finite-criterion}
  \textsf{all-good}\iff\smin>0,\qquad
  \textsf{mixed}\iff\smin\le0<\smax,\qquad
  \textsf{no-good}\iff\smax\le0 .
\end{equation}
\end{theorem}

A tie $s=0$ is on the bad side, matching \eqref{eq:acc}.  The criterion turns
the phase of one instance into two extrema over a compact set.

\begin{proposition}[The mixed phase is a degeneracy question]\label{rem:mixed-interval}
For fixed $(C,\Omega)$ the phases partition the threshold axis:
all-good on $(0,\Aminus]$, mixed on $(\Aminus,\Aplus]$, no-good on
$(\Aplus,1)$.  Hence
\begin{equation}\label{eq:mixed-iff}
  \exists\,\tau\in(0,1)\text{ with }(C,\Omega)\ \textsf{mixed}
  \iff
  \Acc_T(\cdot\,;C)\text{ is non-constant on }\{F_B:B\in\Eg\}.
\end{equation}
If the held-out accuracy of the minimum-norm response is almost surely
constant along a family, then $\pmix\equiv0$ on that family and its profile
lies on the edge of $\simplex$ joining $(1,0,0)$ to $(0,0,1)$ for every $\tau$.
\end{proposition}

Two distinct responses of equal accuracy do not make an instance mixed, so
\eqref{eq:mixed-iff} is strictly stronger than non-uniqueness of the
minimum-norm response.  A spread of accuracies across finite-training runs is
not evidence of it either: such runs need not be global minimisers, and
their spread is typically removable by changing the optimiser or regulariser.
The natural source of the required degeneracy would be a nontrivial
stabiliser of the observed data inside the isomorphism group of
\Cref{def:quotient}, which by \Cref{thm:quotient-invariance} carries one
global response to another; we expect that stabiliser to be trivial with
high probability under the uniform mask law, and responses that respect a
symmetry produce score ties, which \eqref{eq:acc} counts as errors.

\subsection{\texorpdfstring{$L_2$}{L2}-regularised biases}
\label{app:bias}

Add a hidden bias $b\in\Rbb^{h_q}$ and an output bias $c\in\Rbb^q$, so that
$\theta^{\mathrm b}=(E_U,E_V,W,A,b,c)$ and
\begin{equation}\label{eq:biased-network}
  F_{\theta^{\mathrm b}}(u,v)=A\,\relu\!\bigl(W(E_U(u)+E_V(v))+b\bigr)+c,
  \qquad
  \lVert\theta^{\mathrm b}\rVert^2=\lVert\theta\rVert^2+\euc b^2+\euc c^2,
\end{equation}
and let $\mathcal S_q^{\mathrm b}(\Dset)$ be the minimum-norm set under the
constraints of \eqref{eq:minnorm}.  Superscript $\mathrm b$ marks the
corresponding phase events, probabilities, and profile.  With $Z=W\Emb$ and
$\mathcal F_{\mathrm b}$ the set of $(A,Z,b,c)$ satisfying the observed unit
margins for $F(x)=A\relu(Z\xicell_x+b)+c$, define
\begin{equation}\label{eq:biased-reduced-objective}
  \Phi_{\mathrm b}(A,Z,b,c)=\nucnorm Z+\tfrac12\bigl(\frob A^2+\euc b^2+\euc c^2\bigr).
\end{equation}

\begin{lemma}[Bias complexity bridge]\label{lem:bias-complexity-bridge}
The biased problem is equivalent, in optimal value and in score responses,
to $\min_{\mathcal F_{\mathrm b}}\Phi_{\mathrm b}$.  There is an absolute
constant $C_{\mathrm b}$ such that on every addition mask every biased
minimum-norm response $F$ has a finite-width bias-free representation
$F(x)=\widehat A\relu(\widehat Z\xicell_x)$ with
$\frob{\widehat A}\nucnorm{\widehat Z}\le C_{\mathrm b}q^{7/6}(\log q)^{1/3}$.
\end{lemma}

The representation is an analysis device whose width may exceed $h_q$ by one
unit, which \Cref{lem:one-sided-notch} permits.

\begin{corollary}[Transfer to $L_2$-regularised biases]\label{cor:l2-bias-transfer}
For the model \eqref{eq:biased-network}:
\begin{enumerate}[label=(\roman*),leftmargin=2.2em,itemsep=2pt,topsep=2pt]
  \item feasibility, phase decomposition, and quotient invariance hold, so
        $\Pi^{\mathrm b}_{q,\tau}([C])$ is well defined on $\Qallq$;
  \item for fixed $\alpha,\tau\in(0,1)$,
        $\min_{\theta^{\mathrm b}\in\mathcal S_q^{\mathrm b}(\Dset(\Cadd,\Omega))}
        \Acc_T(\theta^{\mathrm b};\Cadd)\xrightarrow{\mathrm P}1$,
        $P^{\forall,\mathrm b}_{q,\tau}(\Madd)\to1$, and
        $\Pi^{\mathrm b}_{q,\tau}([\Cadd])\to(1,0,0)$;
  \item uniformly over deterministic biased selectors,
        $P^{\mathsf A,\mathrm b}_{q,\tau}(\Munif)\le(2e^2)^Nq^{-\tau(1-\alpha)N}\to0$;
        hence $P^{\forall,\mathrm b}_{q,\tau}(\Munif)\to0$ and biased
        all-good colourings have zero uniform mass;
  \item on a finite instance, with $s_{-,\mathrm b},s_{+,\mathrm b}$ the
        extrema of the $\lceil\tau\card T\rceil$-th largest held-out margin
        over the biased reduced minimiser set, the three phases are
        characterised by the sign conditions of \eqref{eq:finite-criterion};
        \Cref{thm:classification-readiness} holds for the biased profiles.
\end{enumerate}
\end{corollary}

\subsection{Unstructured colourings}
\label{app:uniform-counting}

Removing the Latin constraint gives an elementary independent-label bound for
comparison with \Cref{lem:uniform-selector-count}.

\begin{proposition}[Unstructured colourings]\label{prop:general-colouring}
Let $C\sim\Unif(\Cq)$.  For any learner, any $\alpha\in(0,1)$, and any
$\tau>1/q$,
\begin{equation}\label{eq:general-colouring}
  \Prob_{C\sim\Unif(\Cq),\,\Omega\sim\nu_q}
  \bigl[\Acc_T(\mathsf A_q(\Dset(C,\Omega));C)\ge\tau\bigr]
  \le\exp\bigl(-2(\tau-1/q)^2\card T\bigr).
\end{equation}
\end{proposition}

\begin{proof}
Conditionally on $\Dset$, the labels on $T$ are independent and uniform, so a
strict prediction is correct with probability at most $1/q$ per cell,
independently; Hoeffding's inequality gives the bound.
\end{proof}

\subsection{A proof-guided observation-scale question}
\label{app:proof-guided-scaling}

Tracking the discarded-cell parameter $\eta$ through the atlas calculation of
\Cref{app:notch-proof} suggests that the notch argument tolerates
\begin{equation}\label{eq:alpha-floor}
  \alpha_q\gg\frac{(\log q)^3}{\sqrt q},
  \qquad\text{equivalently}\qquad
  n_q\gg q^{3/2}(\log q)^3.
\end{equation}
Turning this into a theorem requires a uniform accounting of the
$\eta$-dependence in the covering and concentration bounds; the
$O(q\operatorname{polylog}q)$ scale is a separate target.

\subsection{A formal target for complete asymptotic classification}
\label{app:classification-readiness}

Fix $\tau\in(0,1)$.  Let $\Kinf$ be the set of subsequential limits of
$\profile([C_{q_j}])$ over sequences of orders $q_j\to\infty$ and classes
$[C_{q_j}]\in\mathcal Q_{q_j}$.  Set
$\vertexV=(1,0,0)$ and $\edgeE=\{(0,\lambda,1-\lambda):0\le\lambda\le1\}$,
and define the uniform polarisation defect
\begin{equation}\label{eq:defect}
  \Udefect=\sup_{[C]\in\Qq}\pall([C])\bigl(1-\pall([C])\bigr).
\end{equation}
\emph{Full edge realisation} holds when for every $\lambda\in[0,1]$ there are
$q_j\to\infty$ and classes with $p^{\forall}_{q_j,\tau}([C_{q_j}])\to0$ and
$p^{\mathrm{mix}}_{q_j,\tau}([C_{q_j}])\to\lambda$.

\begin{theorem}[Classification readiness]\label{thm:classification-readiness}
\Cref{thm:main}\ref{it:main-add} gives $\vertexV\in\Kinf$.  Then
$\Kinf=\{\vertexV\}\cup\edgeE$ holds if and only if both $\Udefect\to0$ and
full edge realisation holds.  Under these two conditions a convergent profile
sequence tends to $\vertexV$ iff $\pall\to1$, and to $(0,\lambda,1-\lambda)$
iff $\pall\to0$ and $\pmix\to\lambda$.
\end{theorem}

The theorem packages the remaining programme as a polarisation estimate and a
realisation construction.  Realisation is the harder half:
\Cref{rem:mixed-interval} says it needs instances on which $\Acc_T$ is
non-constant over the global responses, and none is known.

%% file: appendix/01-proofs.tex
\section{Proofs}\label{app:proofs}

This appendix gives the detailed proofs for
\Cref{sec:model}, \Cref{thm:main}, and \Cref{app:supp-theory}; the abelian extension
\Cref{cor:abelian-allgood} is proved in \Cref{app:abelian-carriers}.
\Cref{app:prelim-proofs} covers
the modelling section, \Cref{app:struct-proofs} the invariance theorem and
the finite-instance criterion, \Cref{app:sep-proofs} the two poles of
\Cref{thm:main}, \Cref{app:bias-proofs} the bias transfer, and
\Cref{app:external} states the external results in the forms used.

\subsection{Proofs for \texorpdfstring{\Cref{sec:model}}{the modelling section}}
\label{app:prelim-proofs}

\begin{proof}[Proof of \Cref{thm:bipartite-model}]
Assigning the value $C(u,v)$ to $U_uV_v$ is a bijection between label tables and
edge colourings with colour set $\qset$, for every integer $q\ge2$.

Now suppose $q$ is a prime power and identify $\qset$ with $\mathbb F_q$.
For $a\in\mathbb F_q$, define
\begin{equation}
  \ell_a(t)=1-(t-a)^{q-1}.
\end{equation}
The finite-field identity $z^{q-1}=1$ for $z\ne0$ gives
$\ell_a(t)=\mathbf 1\{t=a\}$ for $t\in\mathbb F_q$. Hence
\begin{equation}\label{eq:polynomial-interpolant}
  P_C(X,Y)
  =\sum_{u,v\in\mathbb F_q}
    C(u,v)\,\ell_u(X)\ell_v(Y)
\end{equation}
has degree at most $q-1$ in each variable and agrees with $C$ on
$\mathbb F_q^2$. The vector space of such reduced polynomials and the vector
space of functions $\mathbb F_q^2\to\mathbb F_q$ both have dimension $q^2$.
Since \eqref{eq:polynomial-interpolant} makes the evaluation map surjective, it
is bijective, proving uniqueness.

Suppose $C$ is a Latin colouring. Fixing a left endpoint $U_u$, the colours
$C(u,0),\dots,C(u,q-1)$ are pairwise distinct because each symbol appears once in
row $u$; fixing a right endpoint gives the same statement for columns. Hence
edges sharing an endpoint receive distinct colours, which is properness.

Conversely, properness says the $q$ edges at each left vertex, and at each right
vertex, carry pairwise distinct colours. Since there are exactly $q$ colours
available, each colour occurs exactly once in every row and every column, which
is the Latin condition.

Fix a colour $a$. The Latin condition gives each left vertex and each right
vertex exactly one incident edge of colour $a$, so $M_a$ is a perfect matching.
The classes are pairwise disjoint and cover $E(K_{q,q})$ by construction.
Conversely, if the edge set decomposes into $q$ perfect matchings then each
colour occurs exactly once at each vertex, which is properness and hence the
Latin condition.

Under each bijection, the mask selects the same evaluations, cells, or edges.
\end{proof}

\begin{proof}[Proof of \Cref{thm:phase-profile}: non-emptiness]
We exhibit a feasible point of finite norm; attainment then follows from
closedness and coercivity. Let $e_0,\dots,e_{2q}$ be the standard basis of
$\Rbb^{d_q}$, $d_q=2q+1$, and set
\begin{equation}
  E_U(u)=e_u+\tfrac12e_{2q},
  \qquad
  E_V(v)=e_{q+v}+\tfrac12e_{2q},
\end{equation}
so that $E_U(u)+E_V(v)=e_u+e_{q+v}+e_{2q}$. For each cell $(i,j)$ place one
hidden row
\begin{equation}
  w_{ij}=e_i+e_{q+j}-\tfrac32e_{2q}.
\end{equation}
Its preactivation at cell $(u,v)$ is
$\mathbf 1\{u=i\}+\mathbf 1\{v=j\}-\tfrac32$, which equals $\tfrac12$ at
$(i,j)$ and is at most $-\tfrac12$ elsewhere. Hence after $\relu$ the unit is
active only at $(i,j)$, with value $\tfrac12$. Let this unit output with weight
$2$ to class $C(i,j)$ and $0$ to every other class. The resulting network has
score $1$ at the true class of every cell and $0$ at every other class, so the
margin constraint of \eqref{eq:minnorm} is met with equality on the whole grid,
in particular on $\Omega$. There are $q^2=h_q$ such units, so the width is
respected and the norm is finite.

The feasible set of \eqref{eq:minnorm} is closed, being an intersection of
preimages of closed sets under the continuous map $\theta\mapsto F_\theta$, and
$\theta\mapsto\tfrac12\lVert\theta\rVert^2$ is coercive. A coercive continuous
function on a non-empty closed set attains its infimum, so
$\Sq(\Dset)\ne\varnothing$.
\end{proof}

\subsection{Proofs of invariance and of the finite-instance criterion}
\label{app:struct-proofs}

\begin{proof}[Proof of \Cref{thm:quotient-invariance}]
Consider first the direct case of \Cref{def:quotient}. Given $\Omega'$, set
\begin{equation}
  \Omega=\{(\rho(u),\kappa(v)):(u,v)\in\Omega'\},
\end{equation}
a bijection of the space of subsets of a fixed size, hence measure-preserving for
$\nu_q$. Given $\theta=(E_U,E_V,W,A)$ define
\begin{equation}
  E_U'(u)=E_U(\rho(u)),\qquad
  E_V'(v)=E_V(\kappa(v)),\qquad
  W'=W,
\end{equation}
and let $A'$ have rows $A'_{\gamma(a),:}=A_{a,:}$ for $a\in\qset$. Then
\begin{equation}\label{eq:score-transport}
  F_{\theta',\gamma(a)}(u,v)=F_{\theta,a}(\rho(u),\kappa(v)).
\end{equation}
Each transformation permutes the entries of a finite table or the rows of a
matrix, so $\lVert\theta'\rVert=\lVert\theta\rVert$. By
\eqref{eq:score-transport} the training constraints correspond term by term, and
applying the same construction with the inverse permutations yields the inverse
map; hence we have a norm-preserving bijection between feasible sets, which must
map global minimisers onto global minimisers. At a test cell the true label is
transported by the same $\gamma$, so strict argmax correctness is preserved cell
by cell, and therefore so is $\Acc_T$.

For the transposed case set
$\Omega=\{(\kappa(v),\rho(u)):(u,v)\in\Omega'\}$ and
\begin{equation}
  E_U'(u)=E_V(\rho(u)),\qquad
  E_V'(v)=E_U(\kappa(v)),
\end{equation}
with $A'$ permuted by $\gamma$ as before. The new preactivation at $(u,v)$ equals
the old one at $(\kappa(v),\rho(u))$. The two embedding tables enter
\eqref{eq:network} symmetrically, so swapping them is again norm-preserving, and
the uniform mask law is invariant under these coordinate permutations. Hence all
four events and their probabilities are unchanged.
\end{proof}

\begin{proof}[Proof of \Cref{thm:phase-profile}: decomposition]
By the first part the set $\Sq(\Dset)$ is non-empty. If it contains no good
global minimiser then $\Aplus<\tau$, which is \textsf{no-good}. Otherwise
$\Aplus\ge\tau$; if additionally there is no bad global minimiser then all of them
are good, which is \textsf{all-good}, and if there is one then good and bad coexist,
which is \textsf{mixed}. The three cases are mutually exclusive and exhaustive.
Taking probabilities under $\nu_q$ gives $\pall+\pmix+\pnone=1$.  The
threshold-axis partition of \Cref{rem:mixed-interval} is the same statement
read for fixed $(C,\Omega)$ and varying $\tau$, and \eqref{eq:mixed-iff}
follows from \Cref{thm:finite-phase-criterion}, since the global responses
are exactly $\{F_B:B\in\Eg\}$.
\end{proof}

\begin{proof}[Proof of \Cref{lem:reduced-normalization}]
If $\Omega=\varnothing$, the zero parameters are the unique minimiser and
all stated quantities vanish.  Assume henceforth that $\Omega\ne\varnothing$.
For any factorisation $Z=W\Emb$ the nuclear-norm factorisation identity gives
\begin{equation}\label{eq:factor-identity}
  \frac{\frob{W}^2+\frob{\Emb}^2}{2}\ \ge\ \nucnorm{W\Emb}=\nucnorm{Z},
\end{equation}
with equality attained as follows: if $Z=U\Sigma V^{\top}$ is a compact SVD, take
$W=U\Sigma^{1/2}$ and $\Emb=\Sigma^{1/2}V^{\top}$, so that
$\frob{W}^2=\frob{\Emb}^2=\operatorname{tr}\Sigma=\nucnorm{Z}$. Since
$\rank Z\le2q<d_q=2q+1$, this factorisation respects the embedding dimension, and
padding with zero rows respects the hidden width.

Now fix a reduced representation $(A,Z)$. Positive homogeneity of $\relu$ gives,
for every $t>0$,
\begin{equation}
  A\relu(Z\xicell_x)=(t^{-1}A)\relu\bigl((tZ)\xicell_x\bigr),
\end{equation}
so the score table, and hence feasibility, is unchanged under
$(A,Z)\mapsto(t^{-1}A,tZ)$. Eliminating $W$ and $\Emb$ via
\eqref{eq:factor-identity}, the objective of \eqref{eq:minnorm} at the balanced
point becomes
\begin{equation}
  \phi(t)=t\,\nucnorm{Z}+\frac{\frob{A}^2}{2t^2}.
\end{equation}
Then $\phi'(t)=\nucnorm{Z}-\frob{A}^2/t^3$, so the unique stationary point is
$t^3=\frob{A}^2/\nucnorm{Z}$, it is a minimum since $\phi''>0$, and
\begin{equation}
  \min_{t>0}\phi(t)=\tfrac32\bigl(\frob{A}\,\nucnorm{Z}\bigr)^{2/3}.
\end{equation}
Minimising over all feasible reduced representations gives $\Vstar=\tfrac32\Jstar^{2/3}$. At
the optimal $t$, $\frob{A}^2$, $\frob{W}^2$, and $\frob{\Emb}^2$ all equal
$\Jstar^{2/3}$, which is the balanced factorisation.
\end{proof}

\begin{proof}[Proof of \Cref{thm:finite-phase-criterion}]
If $\Omega=\varnothing$, then $g=0$, $\Eg=\nucsphere$, and every $A_B=0$;
the only global response is zero and the claim follows directly.  Assume
$\Omega\ne\varnothing$ below.
Positive homogeneity normalises any non-zero reduced representation $(A,B)$ to
$(\nucnorm{B}A,\ B/\nucnorm{B})$ without changing either the scores or the product
$\frob{A}\nucnorm{B}$. Hence the minimal reduced product equals
$\inf_{B\in\nucsphere}\rho(B)$.

The set $\nucsphere$ is a compact subset of a finite-dimensional space. The
function $\rho$ is lower semicontinuous: if $B_i\to B$ with
$\liminf\rho(B_i)<\infty$, pass to a subsequence along which the fixed-list
minimisers $A_i$ are bounded and convergent; continuity of the constraints in
\eqref{eq:rhoB} shows the limit is feasible for $B$, so
$\rho(B)\le\liminf\rho(B_i)$. The cell-selector construction of
\Cref{thm:phase-profile} provides at least one feasible list, so $g<\infty$ and the
minimum is attained. As $\rho\ge g$ everywhere on $\nucsphere$, we may write
$\Eg=\{B:\rho(B)\le g\}$, which is closed in a compact set, hence compact.

For fixed $B$ the feasible $A$ form a closed convex set and $\frob{A}^2$ is
strictly convex, so the minimiser $A_B$ is unique. If $B_i\to B$ with all
$B_i\in\Eg$, then $\frob{A_{B_i}}=g$ for all $i$; any convergent subsequence has
a limit that is feasible at $B$ with norm $g$, hence equals $A_B$ by uniqueness.
Therefore the whole sequence converges and $B\mapsto A_B$ is continuous; $\relu$
is continuous, so $B\mapsto F_B$ is continuous. A finite-order statistic of
continuous functions is continuous, so the extrema $\smin,\smax$ in
\eqref{eq:s-extrema} are attained.

Any reduced global minimiser, once normalised, must have $B\in\Eg$, and its output
must be $A_B$: otherwise replacing it by the smaller-norm fixed-list output would
strictly decrease the product, contradicting optimality. Conversely every
$B\in\Eg$ paired with $A_B$ attains the minimal product and lifts, by the
balancing step of \Cref{lem:reduced-normalization}, to a global minimiser of the
original problem. Hence $\{F_B:B\in\Eg\}$ is exactly the set of global responses.

Finally, a response has strict test accuracy at least $\tau$ iff at least
$k=\lceil\tau\card{T}\rceil$ of its held-out margins are strictly positive, iff
$s(B)>0$. Therefore all global responses are good iff $\min_{\Eg}s>0$; good and bad
coexist iff the minimum is non-positive while the maximum is positive; and no
global is good iff the maximum is non-positive. An exact tie $s=0$ counts as bad
by the strict convention in \eqref{eq:acc}. This is \eqref{eq:finite-criterion}.
\end{proof}

\subsection{Proofs for \texorpdfstring{\Cref{sec:main}}{the main theorem}}
\label{app:sep-proofs}

\subsubsection{The addition carrier}\label{app:carrier-proof}

\begin{proof}[Proof of \Cref{lem:addition-carrier}]
We give the construction explicitly. Put
\begin{equation}
  h=\tfrac23,\qquad
  \epsilon=\frac{\sqrt2h^2}{4}=\frac{\sqrt2}{9},\qquad
  \rho=\frac{1-\epsilon}{2},\qquad
  c=\frac{2}{1-\epsilon},
\end{equation}
and
\begin{equation}
  L=\left\lceil\frac{2(\log(q-1)+1)}{(\rho+1/(q-1))^2}\right\rceil=O(\log q).
\end{equation}

\paragraph{Frequency selection.}
We claim there are non-zero frequencies $k_1,\dots,k_L\in\Zbb_q$ such that every
$e\ne0$ satisfies
\begin{equation}\label{eq:freq-cover}
  \frac1L\sum_{\ell=1}^{L}\cos(2\pi k_\ell e/q)\ \le\ \rho .
\end{equation}
Argue probabilistically: for uniform non-zero $k$, the mean of
$\cos(2\pi ke/q)$ is $-1/(q-1)$, since the full sum over all $k\in\Zbb_q$
vanishes. The summands lie in $[-1,1]$, so Hoeffding's inequality
\citep{hoeffding1963probability} bounds the probability that
\eqref{eq:freq-cover} fails at a fixed $e$ by
$\exp\{-L(\rho+1/(q-1))^2/2\}$. A union bound over the $q-1$ values of $e$ gives
total probability below $1$ by the choice of $L$, so the required frequencies
exist.

\paragraph{An exact ReLU surrogate for multiplication.}
Let $Q_h$ be the piecewise-linear interpolant of $t^2$ on the grid
$\{-2,-\tfrac43,-\tfrac23,0,\tfrac23,\tfrac43,2\}$. It has the exact ReLU
representation
\begin{equation}\label{eq:Qh}
  Q_h(t)=h\lvert t\rvert+2h\sum_{j=1}^{2}\pospart{\lvert t\rvert-jh},
  \qquad
  \lvert t\rvert=\pospart{t}+\pospart{-t},
\end{equation}
and satisfies $0\le Q_h(t)-t^2\le h^2/4$ for $\lvert t\rvert\le2$. Hence
\begin{equation}
  P(x,y)=\frac{Q_h(x+y)-Q_h(x-y)}{4}
\end{equation}
satisfies, for $\lvert x\rvert,\lvert y\rvert\le1$,
\begin{equation}\label{eq:prod-approx}
  \lvert P(x,y)-xy\rvert\le h^2/16 .
\end{equation}

\paragraph{Scores.}
For each frequency $k_\ell$, set
$\phi_\ell(t)=2\pi k_\ell t/q$ and use \eqref{eq:prod-approx} inside the
addition formulas:
\begin{align}
  c_\ell(u,v)
    &=P(\cos\phi_\ell(u),\cos\phi_\ell(v))
      -P(\sin\phi_\ell(u),\sin\phi_\ell(v)),\\
  s_\ell(u,v)
    &=P(\sin\phi_\ell(u),\cos\phi_\ell(v))
      +P(\cos\phi_\ell(u),\sin\phi_\ell(v)),
\end{align}
so that $\lvert c_\ell-\cos\phi_\ell(u+v)\rvert\le h^2/8$ and
$\lvert s_\ell-\sin\phi_\ell(u+v)\rvert\le h^2/8$. Define
\begin{equation}\label{eq:carrier-score}
  \Ccarc{a}(u,v)
  =\frac cL\sum_{\ell=1}^{L}
   \bigl(\cos\phi_\ell(a)\,c_\ell(u,v)
        +\sin\phi_\ell(a)\,s_\ell(u,v)\bigr).
\end{equation}

\paragraph{Margins.}
The idealised version of \eqref{eq:carrier-score}, with $c_\ell,s_\ell$ replaced
by exact cosines and sines, scores $c$ at $a=u+v$ and, by \eqref{eq:freq-cover},
at most $c\rho$ at any wrong class. Each realised score differs from its ideal by
at most $c\sqrt2h^2/8$, so every truth--rival gap is at least
$c(1-\rho-\sqrt2h^2/4)=1$, which is the lower bound in
\eqref{eq:carrier-margins}. For the upper bound, each ideal score is at most $c$
in absolute value and each realised score deviates by at most $c\epsilon/2$; since
$\epsilon=\sqrt2/9<1/6$ and $c=2/(1-\epsilon)<12/5$, every gap is at most
\begin{equation}
  2c\Bigl(1+\frac\epsilon2\Bigr)<\frac{24}{5}\cdot\frac{13}{12}=\frac{26}{5}<9 .
\end{equation}

\paragraph{Additivity and width.}
Every hinge in \eqref{eq:Qh} is the positive part of $r_u+s_v-t$. The constant $t$
may be split, half into the row potential and half into the column potential, so
each is a legitimate additive preactivation in the reduced coordinates $\xicell_{(u,v)}=e_u+e_{q+v}$ of \Cref{sec:model-endpoint}.
Each frequency needs only a constant number of hidden rows, so the total width is
$O(L)=O(\log q)$, which is below $q^2$ for large $q$.

\paragraph{Norms.}
Each preactivation row has Euclidean norm $O(\sqrt q)$ over the $2q$ row/column
potential coordinates, and each corresponding output column has Euclidean norm
$O(\sqrt{q}/L)$. Hence
\begin{equation}
  \nucnorm{Z}\le\sqrt{O(L)}\,\frob{Z}=O(L\sqrt q),
  \qquad
  \frob{A}=O(\sqrt{q/L}),
\end{equation}
so, taking these constructed factors as $(A_{\rm car},Z_{\rm car})$,
$J_{\rm car}:=\frob{A_{\rm car}}\nucnorm{Z_{\rm car}}
\le C_0q\sqrt{\log q}$ after fixing an absolute $C_0$, which is
\eqref{eq:carrier-margins}. A compact SVD of
$Z_{\rm car}$ finally yields admissible $W$ and $\Emb$ as in
\Cref{lem:reduced-normalization}.
\end{proof}

\subsubsection{The one-sided notch lemma}\label{app:notch-proof}

\begin{proof}[Proof of \Cref{lem:one-sided-notch}]
Throughout, $y(x)=\Cadd(x)$.

\paragraph{Step 0: the design Gram matrix.}
The full addition design satisfies
\begin{equation}\label{eq:gram}
  \opnorm{\sum_{x\in\Xq}\xicell_x\xicell_x^{\top}}=2q,
\end{equation}
because that Gram matrix is
$\begin{psmallmatrix}qI&\mathbf 1\mathbf 1^{\top}\\ \mathbf 1\mathbf 1^{\top}&qI\end{psmallmatrix}$.

\paragraph{Step 1: low-rank split.}
Let $m=\lceil\sqrt q\rceil$ and split $Z$ into its rank-$m$ SVD core $Z_m$ and
the tail. Write the $i$-th column of $A$ as $a_i$ and the $i$-th row of $Z_m$ as
$z_i^{(m)}$, and put the non-zero terms in the form
\begin{equation}
  A\relu(Z_m\xicell_x)=\sum_i s_iu_i\relu(w_i^{\top}\xicell_x),
\end{equation}
with $\euc{u_i}=\euc{w_i}=1$ and $s_i=\euc{a_i}\euc{z_i^{(m)}}$, so that
\begin{equation}
  \sum_i s_i\le\frob{A}\frob{Z_m}\le\frob{A}\nucnorm{Z}\le B_q .
\end{equation}
By \eqref{eq:gram} each normalised scalar gate has full-grid squared norm at most
$2q$, so Maurey sampling in Hilbert space \citep{pisier1981remarques} yields a
proxy built from $K_0$ core atoms with full-grid squared vector error at most
$2qB_q^2/K_0$.

The $i$-th tail atom is
$a_i\bigl(\relu(z_i^{\top}\xicell_x)-\relu((z_i^{(m)})^{\top}\xicell_x)\bigr)$.
Using that $\relu$ is $1$-Lipschitz, Cauchy--Schwarz, and the singular-value tail
bound $\frob{Z-Z_m}^2\le\nucnorm{Z}^2/(m+1)$ --- which follows from
$(m+1)\sigma_{m+1}\le\sum_{j\le m+1}\sigma_j\le\nucnorm{Z}$ and termwise
summation --- we get
\begin{equation}\label{eq:tail-bound}
  \Bigl(\sum_i\lVert\text{tail atom}_i\rVert_{\ell_2(\Xq;\Rbb^q)}\Bigr)^2
  \le2q\frob{A}^2\frob{Z-Z_m}^2
  \le\frac{2qB_q^2}{m+1},
\end{equation}
and sampling $K_1$ tail atoms divides the right-hand side by $K_1$.

\paragraph{Step 2: sampling budgets.}
Set
\begin{equation}
  \epsilon_q=\frac{B_q^2\log q}{q^{5/2}},
  \qquad
  r_q=\begin{cases}
    \min\{\log q,\ \epsilon_q^{-1/2}\},&\epsilon_q>0,\\
    \log q,&\epsilon_q=0,
  \end{cases}
\end{equation}
and
\begin{equation}\label{eq:budgets}
  K_0=\max\{1,\lceil r_qB_q^2/q\rceil\},
  \qquad
  K_1=\max\{1,\lceil r_qB_q^2/(qm)\rceil\}.
\end{equation}
Then $r_q\to\infty$ and $r_q\epsilon_q\to0$ by the hypothesis $\epsilon_q\to0$, and the
combined full-grid squared vector error of the two proxies is $O(N/r_q)=o(N)$.

Write the core proxy as $\widetilde D_0=B_0H_0$, dividing each sampled core scalar
gate by the fixed constant $\sqrt{2q}$ and multiplying that constant into the
output column. The Maurey construction gives
\begin{equation}\label{eq:aq}
  \frob{B_0}^2\le a_q^2:=C_1N/r_q,
\end{equation}
and the normalised feature table satisfies
$\sum_{x\in\Xq}\euc{H_0(x)}^2\le K_0$, so after discarding at most $\eta N$ cells,
\begin{equation}\label{eq:bq}
  \euc{H_0(x)}^2\le b_q^2:=K_0/(\eta N).
\end{equation}
The tail proxy remains explicit outside the covering theorem:
\begin{equation}\label{eq:tail-proxy}
  \widetilde T(x)=\sum_{j=1}^{K_1}v_j
  \bigl(\relu(r_j^{\top}\xicell_x)-\relu((r_j')^{\top}\xicell_x)\bigr),
\end{equation}
which avoids hiding a mask-adaptive component inside a fixed-feature cover.

\paragraph{Step 3: a mask-independent cover code for the core.}
Apply \Cref{ext:zhang} with $p=r=2$, using the block-diagonal lift over the
$q$ output columns in \Cref{ext:ledent}; that is,
represent one evaluation $(x,c)$ by the feature
\begin{equation}
  \zeta_{x,c}=e_c\otimes H_0(x)\in\Rbb^{qK_0}.
\end{equation}
On the retained cells $\euc{\zeta_{x,c}}\le b_q$, while the stacked output weight
has Euclidean norm $\frob{B_0}\le a_q$. Fix precision $\delta=1/16$. Zhang's proof
first fixes a deterministic threshold grid
$\{\theta_s\}_{s=0}^{L_q^{\rm grid}}\subset
[-a_qb_q-\delta/2,\ a_qb_q+\delta/2]$ with
$L_q^{\rm grid}\le\lceil4a_qb_q/\delta+2\rceil$; this grid is determined by
$q,\eta,B_q,r_q$ and is independent of the numerical feature table and mask.

Dispose first of the degenerate branch: if $\delta>a_qb_q$ then Cauchy--Schwarz
bounds every core score coordinate by $a_qb_q<\delta$, so the single zero centre
is already a valid cover; add it to the atlas as a deterministic code. Assume
henceforth $\delta\le a_qb_q$.

At $p=2$ the duality map in Zhang's argument is linear, so each cover centre is
determined by a finite non-negative integer multiplicity code
$\kappa=(\alpha_{x,c,s},\beta_{x,c,s})$ of total mass at most
\begin{equation}\label{eq:Mq}
  M_q:=\bigl\lceil36a_q^2b_q^2\delta^{-2}\bigr\rceil,
\end{equation}
defined on the ambient address set
$\Xq\times\qset\times\{0,\dots,L_q^{\rm grid}\}\times\{+,-\}$ after zero-filling unused
addresses. Zhang's augmented separating vector is
\begin{equation}\label{eq:code-centre}
\begin{aligned}
  y_\kappa(H_0)&=2\sum_{x,c,s}(\alpha_{x,c,s}-\beta_{x,c,s})\,\zeta_{x,c},\\
  a_qz_\kappa&=\frac{2}{a_q}\sum_{x,c,s}(-\alpha_{x,c,s}+\beta_{x,c,s})\,\theta_s ,
\end{aligned}
\end{equation}
where $z_\kappa$ depends only on the code and the deterministic grid. Zhang's
separation argument shows every code that actually produces a cover centre has
$z_\kappa>0$; in the enumeration we simply discard codes with $z_\kappa\le0$, so
no division by zero or sign ambiguity arises. The effective centre is
$\widehat w_\kappa(H_0)=y_\kappa(H_0)/z_\kappa$. If $a_q=0$ the core output
vanishes identically and the zero centre suffices, so the formula is only used
when $a_q>0$; a sampled gate that vanishes on the whole grid contributes a zero
feature and requires no division by a feature norm. By
\Cref{ext:zhang,ext:ledent}, every realised $B_0H_0$ is approximated
within $\delta$ at every output coordinate of every retained cell by some
effective code centre.

The number of codes is likewise fixed before the mask is drawn. The address count
is at most $L_q^{\mathrm{addr}}=2Nq(L_q^{\rm grid}+1)$, and the number of non-negative
integer codes of total mass at most $M_q$ is at most
$(L_q^{\mathrm{addr}}+1)^{M_q}$. From \eqref{eq:budgets}--\eqref{eq:bq},
\begin{equation}
  a_q^2b_q^2=O_\eta\Bigl(\frac{B_q^2}{q}+\frac1{r_q}\Bigr),
\end{equation}
and $a_qb_q$ is at most a fixed polynomial in $q$, so
\begin{equation}\label{eq:code-count}
  \log\#\{\kappa\}
  =O_\eta\Bigl(\Bigl(1+\frac{B_q^2}{q}+\frac1{r_q}\Bigr)\log q\Bigr)=o(N).
\end{equation}
The additive constant $1$ is the at most one extra atom created by the ceiling in
\eqref{eq:Mq}; it contributes $O(\log q)=o(N)$, and in the non-degenerate branch
$a_q^2b_q^2\ge\delta^2$, so it is absorbed by the main term.

\paragraph{Step 4: sign-pattern counting.}
Write the common rank-$m$ row space of the core as $R\in\Rbb^{m\times2q}$ and the
$j$-th sampled core normal as $R^{\top}t_j$, giving $2qm+mK_0$ continuous core
parameters; the tail atoms of \eqref{eq:tail-proxy} contribute at most $5qK_1$
parameters through $v_j\in\Rbb^q$ and $r_j,r_j'\in\Rbb^{2q}$. Hence
\begin{equation}
  D_q^{\mathrm{par}}\le2qm+mK_0+5qK_1
  =O\Bigl(q^{3/2}+\frac{r_qB_q^2}{\sqrt q}+q\Bigr),
\end{equation}
and therefore
\begin{equation}\label{eq:par-count}
  \frac{D_q^{\mathrm{par}}\log q}{N}
  =O\Bigl(\frac{\log q}{\sqrt q}+r_q\epsilon_q\Bigr)=o(1).
\end{equation}

For each $x$ and $j$, the core preactivation $t_j^{\top}R\xicell_x$ is a
degree-$2$ polynomial in the continuous parameters, and the tail preactivations
$r_j^{\top}\xicell_x$, $(r_j')^{\top}\xicell_x$ are degree-$1$.  By
\Cref{ext:caro}, the ternary sign patterns of these $(K_0+2K_1)N$ polynomials
give at most
$\exp(O(D_q^{\mathrm{par}}\log q))$; because zeros are counted, all exact zero
gates are retained.

On a fixed gate stratum each $H_{0,j}(x)$ is a polynomial of degree at most $2$,
so by \eqref{eq:code-centre} the $c$-th score of an effective code centre at cell
$x$ has the explicit form
\begin{equation}\label{eq:centre-poly}
  \widehat F_{\kappa,c}(x)
  =\frac{2}{z_\kappa}\sum_{s,c',\ell}
   \sigma_{\kappa,s,c',\ell}\,H_{0,\ell}(s)H_{0,\ell}(x),
\end{equation}
with integer coefficients $\sigma_{\kappa,s,c',\ell}$ determined entirely by the
code and free of continuous parameters. Thus \eqref{eq:centre-poly} is a
polynomial of degree at most $4$ divided by the positive code-only constant
$z_\kappa$. The tail score \eqref{eq:tail-proxy} is a polynomial of degree at most
$2$ on a fixed tail-gate stratum. Clearing the positive denominator, each
truth--rival half-gap comparison
\begin{equation}
  \widehat F_{\kappa,a}(x)+\widetilde T_a(x)
  -\widehat F_{\kappa,y(x)}(x)-\widetilde T_{y(x)}(x)-\tfrac12
\end{equation}
agrees in sign with a polynomial of degree at most $4$, and there are only
$N(q-1)$ such comparisons. Applying \Cref{ext:caro} again --- or the direct
$3^M$ bound when the number of polynomials is below the number of parameters ---
bounds the number of threshold sets per pair (code, gate stratum) by
$\exp(O(D_q^{\mathrm{par}}\log q))$.

Combining \eqref{eq:code-count} and \eqref{eq:par-count}, the deterministic atlas
$\mathcal H_q$ of all threshold sets arising from gates, zero gates, score ties
and half-gap comparisons satisfies
\begin{equation}\label{eq:atlas}
  \log\card{\mathcal H_q}
  \le\log\#\{\kappa\}+O(D_q^{\mathrm{par}}\log q)=o(N).
\end{equation}
Its address set, threshold grid, codes and polynomial list are all fixed before
$\Omega$ is drawn, whereas the endpoint, the SVD row space, the sampled atoms, the
tail atoms and the effective code may all be chosen after seeing $\Omega$; the
enumeration \eqref{eq:atlas} covers them simultaneously.

\paragraph{Step 5: concentration.}
Let $\Xexc$ contain the high-feature cells discarded in \eqref{eq:bq} together
with the cells where either Maurey approximation exceeds the fixed vector
threshold $1/16$. By Markov's inequality and the total squared errors of Step 2,
choosing $\eta>0$ small and then letting $q\to\infty$,
\begin{equation}
  \card{\Xexc}\le(\eta+o(1))N .
\end{equation}
Off $\Xexc$ the core Maurey error and the tail Maurey error each have vector norm
below $1/16$ and the Zhang centre has entrywise error below $1/16$, so the final
proxy $\bar D$ differs from $D$ by less than $3/16$ in each coordinate and by less
than $3/8$ in each pair gap. Define the atlas threshold set
\begin{equation}
  H_{\bar D}=\Bigl\{x:\max_{a\ne y(x)}\bigl(\bar D_a(x)-\bar D_y(x)\bigr)\ge\tfrac12\Bigr\}.
\end{equation}
Hypothesis~(1) gives $\card{H_{\bar D}\cap\Omega}\le\card{\Xexc}$, while the
$\beta N$ reverse-unit-gap cells of hypothesis~(2) give
$\card{H_{\bar D}}\ge\beta N-\card{\Xexc}\ge\beta N/2$ for large $q$.

Fix any $H\subseteq\Xq$ with $\card{H}\ge\beta N/2$. Then
$\card{H\cap\Omega}$ is hypergeometric with mean at least
$\alpha\beta N/2+O(1)$. Choosing $0<\eta<\alpha\beta/16$, the binomial-coefficient
formula gives
\begin{equation}\label{eq:hypergeom}
  \Prob\bigl(\card{H\cap\Omega}\le2\eta N\bigr)
  =\sum_{j\le2\eta N}
   \frac{\binom{\card H}{j}\binom{N-\card H}{n_q-j}}{\binom{N}{n_q}}
  \le e^{-c_{\alpha,\beta,\eta}N}
\end{equation}
for some $c_{\alpha,\beta,\eta}>0$; the same bound follows from
sampling-without-replacement Chernoff estimates
\citep{hoeffding1963probability}. Taking a union bound over the $\exp(o(N))$ sets
$H$ permitted by \eqref{eq:atlas} leaves total probability $o(1)$. Hence with
probability $1-o(1)$ no such $D$ exists. Exact zero gates and exact ties are
preserved throughout by the ternary sign counting and by the weak threshold
$\tfrac12$.
\end{proof}

\subsubsection{Assembly of part (i)}\label{app:assembly-proof}

The proof of \Cref{thm:main}\ref{it:main-add} is given in full in
\Cref{sec:main-addition}.  For the bias transfer of \Cref{app:bias-proofs}
we record the two gap inequalities in the form used there.  With
$D=9F-\Ccar$, $y=\Cadd(x)$, and
$\gamma_{x,a}=\Ccarc{y}(x)-\Ccarc{a}(x)\in[1,9]$ the carrier gap,
\begin{equation}\label{eq:notch1}
  D_y-D_a=9(F_y-F_a)-\gamma_{x,a}\ge0
  \qquad\text{on observed cells, since }F_y-F_a\ge1,
\end{equation}
and at every held-out non-strict error, including a tie, some $a\ne y$ has
$F_a\ge F_y$, whence
\begin{equation}\label{eq:notch2}
  D_a-D_y=9(F_a-F_y)+\gamma_{x,a}\ge1 .
\end{equation}
The norm budget is $B_q$ of \eqref{eq:JD}.

\subsubsection{The uniform Latin square obstruction}\label{app:counting-proof}

\begin{proof}[Proof of \Cref{lem:uniform-selector-count}(a)]
Since $C$ is uniform on $\Lq$ and independent of $\Omega$,
$P^{\mathsf A}_{q,\tau}(\Munif)
=\Exp_{\Omega}\bigl[\card\{C\in\Lq:\Acc_T(\mathsf A_q(\Dset(C,\Omega));C)\ge\tau\}/\card\Lq\bigr]$,
so it suffices to bound the count for a fixed mask $\Omega$.  The observed
labels take at most $q^{n_q}$ values. Fix one observed dataset $\Dset$. On each test cell the learner has at most one strict
winner; if there is no strict winner, assigning an arbitrary label only enlarges
the candidate set counted below.

Success means the true Latin table differs from the winner table on at most $r$
test cells. Choosing those positions and their replacement labels, the number of
full arrays within that Hamming ball is at most
\begin{equation}
  \sum_{i=0}^{r}\binom{N-n_q}{i}(q-1)^i ,
\end{equation}
and restricting to Latin arrays can only decrease the count. Summing over all
$\Dset$ and averaging over fixed-size masks gives \eqref{eq:selector-count}.
\end{proof}

\begin{proof}[Proof of \Cref{lem:uniform-selector-count}(b)]
By the permanent bound in \Cref{ext:gurvits}, if $B$ is a $q\times q$
non-negative doubly stochastic
matrix then $\per(B)\ge q!/q^q$.

Build the Latin square row by row. When $k$ rows have been placed, the
symbol--column incidence matrix of admissible entries for the next row is
$(q-k)$-regular; dividing by $q-k$ makes it doubly stochastic, so its permanent is
at least
\begin{equation}
  (q-k)^q\,q!/q^q .
\end{equation}
The first row admits $q!$ choices, so
\begin{equation}
  \card{\Lq}\ \ge\ q!\prod_{k=1}^{q-1}\frac{(q-k)^qq!}{q^q}
  =\frac{(q!)^{2q}}{q^{q^2}} .
\end{equation}
Finally $\log(q!)=\sum_{j=1}^{q}\log j\ge\int_1^q\log x\,dx=q\log q-q+1$ gives
$q!\ge(q/e)^q$, and substituting yields $(q/e^2)^{q^2}$.
\end{proof}

\subsubsection{\texorpdfstring{$L_2$}{L2}-regularised biases}
\label{app:bias-proofs}

\begin{proof}[Proof of \Cref{lem:bias-complexity-bridge}]
For fixed $Z=W\Emb$, the nuclear-norm factorisation identity used in
\eqref{eq:factor-identity} gives
\begin{equation}\label{eq:biased-factorisation}
  \inf_{W\Emb=Z}\frac12
  \bigl(\frob{W}^2+\frob{\Emb}^2\bigr)=\nucnorm{Z}.
\end{equation}
Every admissible $Z$ has at most $2q$ columns, while $d_q=2q+1$, so an SVD
factorisation attaining \eqref{eq:biased-factorisation} respects the embedding
dimension; zero padding respects the hidden width.  Keeping $A,b,c$ fixed
therefore establishes both directions of the equivalence with
$\min_{\mathcal F_{\mathrm b}}\Phi_{\mathrm b}$.  Its feasible set is non-empty because the
bias-free cell-selector construction of \Cref{thm:phase-profile} is feasible with
$b=c=0$.  Closedness and coercivity give attainment.

Let $\Ccar$ be the addition carrier and let $J_{\rm car}$ denote the cost of
the constructed representation in \Cref{lem:addition-carrier}.  Balancing its
two reduced factors and taking $b=c=0$ provides a feasible competitor for
every addition mask.  Hence the biased optimum $M_q$ satisfies
\begin{equation}\label{eq:biased-optimum-bound}
  M_q:=\min_{\mathcal F_{\mathrm b}}\Phi_{\mathrm b}
  \le\frac32J_{\rm car}^{2/3}
  \le C_1q^{2/3}(\log q)^{1/3}
\end{equation}
for an absolute constant $C_1$.  Consequently, every reduced global minimiser
obeys
\begin{equation}\label{eq:biased-block-bounds}
  \nucnorm{Z}\le M_q,
  \qquad
  \frob{A},\euc{b},\euc{c}\le\sqrt{2M_q}.
\end{equation}

The hidden bias can be absorbed into the two endpoint blocks.  Put
\begin{equation}
  \overline Z=Z+\tfrac12b\mathbf 1_{2q}^{\top}.
\end{equation}
Since $\xicell_x$ has exactly two unit entries,
$\overline Z\xicell_x=Z\xicell_x+b$.  The added rank-one matrix has nuclear
norm $\sqrt{q/2}\,\euc{b}$, and thus
\begin{equation}\label{eq:hidden-bias-bound}
  \nucnorm{\overline Z}
  \le M_q+\sqrt{qM_q}.
\end{equation}
It follows that
$F_0(x)=A\relu(\overline Z\xicell_x)$ has a bias-free representation of cost
at most
\begin{equation}\label{eq:F0-cost}
  \Jred(F_0)
  \le\sqrt{2M_q}\bigl(M_q+\sqrt{qM_q}\bigr).
\end{equation}

The constant table $F_c(x)=c$ has a one-hidden-unit representation: take
$z_c=\tfrac12\mathbf 1_{2q}^{\top}$, for which
$\relu(z_c\xicell_x)=1$, and use $c$ as the output column.  By
\eqref{eq:biased-block-bounds},
\begin{equation}\label{eq:output-bias-bound}
  \Jred(F_c)\le\sqrt{qM_q}.
\end{equation}
If $F_1,F_2$ have respective product costs $J_1,J_2$, balance their factors,
concatenate the output matrices, and stack the hidden matrices.  The triangle
and Cauchy--Schwarz inequalities give
\begin{equation}\label{eq:reduced-sum-bound}
  \Jred(F_1+F_2)
  \le(\sqrt{J_1}+\sqrt{J_2})^2
  \le2(J_1+J_2).
\end{equation}
Applying \eqref{eq:reduced-sum-bound} to $F=F_0+F_c$ and using
\eqref{eq:biased-optimum-bound} yields
\begin{align}
  \Jred(F)
  &\le2\left[
    \sqrt{2M_q}\bigl(M_q+\sqrt{qM_q}\bigr)+\sqrt{qM_q}
  \right]\\
  &=O\!\left(M_q^{3/2}+\sqrt q\,M_q+\sqrt{qM_q}\right)\\
  &=O\!\left(q^{7/6}(\log q)^{1/3}\right),
\end{align}
which proves the bound of \Cref{lem:bias-complexity-bridge}.
\end{proof}

\begin{proof}[Proof of \Cref{cor:l2-bias-transfer}]
The bias-free feasible set embeds into the biased one by taking $b=c=0$, so
feasibility persists; attainment again follows from closedness and coercivity.
For a direct colouring isomorphism, use the parameter map in the proof of
\Cref{thm:quotient-invariance}, keep $b$ fixed, and permute the coordinates of
$c$ by the same symbol permutation used on the rows of $A$.  Under
transposition, also exchange the two embedding tables.  These maps preserve
\eqref{eq:biased-network}, the constraints, and cellwise strict correctness, and
their inverses have the same form.  This proves the biased phase and quotient
claims.

For addition, let $F$ be any biased minimum-norm response.  By
\Cref{lem:bias-complexity-bridge}, it has an analysis representation with
\begin{equation}
  \Jred(F)=O\!\left(q^{7/6}(\log q)^{1/3}\right).
\end{equation}
Set $D=9F-\Ccar$.  The observed non-negative gaps in \eqref{eq:notch1} and the
unit reverse gaps at every held-out non-strict error in \eqref{eq:notch2}
depend only on the score table, so they remain valid.  Applying
\eqref{eq:reduced-sum-bound} gives
\begin{equation}\label{eq:biased-D-bound}
  \Jred(D)
  \le2\bigl(9\Jred(F)+\Jred(\Ccar)\bigr)
  =O\!\left(q^{7/6}(\log q)^{1/3}\right),
\end{equation}
and therefore
\begin{equation}
  \frac{\Jred(D)^2\log q}{q^{5/2}}
  =O\!\left(q^{-1/6}(\log q)^{5/3}\right)\to0.
\end{equation}
For any fixed $\tau<1$, an endpoint below accuracy $\tau$ supplies more than
$(1-\tau)(1-\alpha)N$ reverse-gap cells.  Taking
$\beta=\tfrac12(1-\tau)(1-\alpha)$, the one-sided notch lemma excludes all such
biased endpoint responses simultaneously with probability $1-o(1)$.  Applying
this statement at each fixed $\tau<1$ proves the first convergence in
\Cref{cor:l2-bias-transfer}(ii); the all-good probability and profile follow from
the phase decomposition.

The proof of \Cref{lem:uniform-selector-count} uses only observation
measurability of the predicted table.  It therefore applies verbatim to every
biased minimum-norm selector and gives \Cref{cor:l2-bias-transfer}(iii).  On a
biased all-good instance every such selector succeeds, which proves
$P_{q,\tau}^{\forall,\mathrm b}(\Munif)\to0$; Markov's inequality then gives
the biased zero-mass statement.

For the finite-instance claim, let
\begin{equation}
  g_{\mathrm b}=\min_{\mathcal F_{\mathrm b}}\Phi_{\mathrm b},
  \qquad
  E_{g,\mathrm b}
  =\{(A,Z,b,c)\in\mathcal F_{\mathrm b}:
       \Phi_{\mathrm b}(A,Z,b,c)=g_{\mathrm b}\}.
\end{equation}
This set is non-empty and compact.  By
\eqref{eq:biased-factorisation}, its score responses are exactly the original
biased global responses.  The held-out margins and their finite order
statistics are continuous on $E_{g,\mathrm b}$.  Taking the minimum and maximum
of the $k$-th largest margin, where $k=\lceil\tau\card T\rceil$, and repeating
the final step of \Cref{thm:finite-phase-criterion} gives the same three sign
conditions.  Finally, the proof of \Cref{thm:classification-readiness} uses
only the simplex identity, the addition vertex, and the definitions of the
polarisation defect and edge realisation.  The biased versions established
above supply these ingredients, completing the transfer.
\end{proof}

\subsubsection{Classification readiness}

\begin{proof}[Proof of \Cref{thm:classification-readiness}]
Assume first that $\Udefect\to0$. Take any convergent profile sequence and let
$t$ be the limit of its first coordinate. By \eqref{eq:defect}, $t(1-t)=0$, so
$t\in\{0,1\}$. If $t=1$ the simplex identity of \Cref{thm:phase-profile} forces
the limit to be $\vertexV$; if $t=0$ the limit lies in $\edgeE$. Hence
$\Kinf\subseteq\{\vertexV\}\cup\edgeE$. \Cref{thm:main}\ref{it:main-add} gives
$\vertexV\in\Kinf$ and full edge realisation gives $\edgeE\subseteq\Kinf$, so
$\Kinf=\{\vertexV\}\cup\edgeE$ holds.

Conversely suppose $\Kinf=\{\vertexV\}\cup\edgeE$ holds and
$\limsup_{q\to\infty}\Udefect>0$. Then there are
$\delta>0$, a sequence of orders $q_j\to\infty$ and classes $[C_{q_j}]$ with
\begin{equation}
  p_{q_j,\tau}^{\forall}([C_{q_j}])
  \bigl(1-p_{q_j,\tau}^{\forall}([C_{q_j}])\bigr)\ge\delta.
\end{equation}
The probability simplex is compact, so we may extract a further subsequence
along which the profile
converges. Its first coordinate lies strictly inside $(0,1)$, so the limit is
neither $\vertexV$ nor a point of $\edgeE$, contradicting $\Kinf=\{\vertexV\}\cup\edgeE$.
Therefore $\Udefect\to0$. Moreover $\Kinf=\{\vertexV\}\cup\edgeE$ says every point of $\edgeE$
has a realising sequence of classes by the definition of $\Kinf$, which is exactly
full edge realisation.

The final coordinate characterisation is immediate from
$\pall+\pmix+\pnone=1$.
\end{proof}

\subsection{External results used}\label{app:external}

We state the covering/counting results, permanent bound, sampling
estimates, and character identities in the forms used in the proofs.

\begin{theorem}[\citealp{zhang2002covering}, Theorem 4]\label{ext:zhang}
Let sample vectors satisfy $\lVert x_i\rVert_p\le b$ and linear weights satisfy
$\lVert w\rVert_r\le a$, where $2\le p<\infty$ and $1/p+1/r=1$. Then for every
$\varepsilon>0$ the empirical $\ell_\infty$ covering number of the linear
function class on these $n$ samples satisfies
\begin{equation}\label{eq:zhang}
  \log_2N_\infty
  \le36(p-1)a^2b^2\varepsilon^{-2}
    \log_2\bigl(2\lceil4ab/\varepsilon+2\rceil n+1\bigr).
\end{equation}
\end{theorem}

\begin{theorem}[\citealp{ledent2021norm}, Proposition 5]\label{ext:ledent}
Let $X\in\Rbb^{n\times U\times d}$ with $\euc{X_{i,u,:}}\le b$ for every patch.
For any reference $M$, the matrix class $\{XA:\frob{A-M}\le a\}$ admits an
entrywise $\ell_\infty$ cover over all sample, patch and output-channel
coordinates with
\begin{equation}\label{eq:ledent}
  \log N\le36a^2b^2\varepsilon^{-2}
  \log_2\bigl(((8ab/\varepsilon)+7)\,mnU\bigr),
\end{equation}
where $m$ is the number of output channels.
\end{theorem}

\Cref{ext:ledent} is the block-diagonal lift of \Cref{ext:zhang} to the index
triple (sample, patch, output channel), which is exactly the form we need in
Step~3 of the proof of \Cref{lem:one-sided-notch}.

\begin{theorem}[\citealp{caro2020pseudo}, Corollary 3.2]\label{ext:caro}
For $M\ge n$ real polynomials of degree at most $d\ge1$ in $n$ real variables, the
number of realisable sign assignments in $\{-1,0,1\}^M$ is at most
\begin{equation}\label{eq:caro}
  (8edM/n)^n .
\end{equation}
\end{theorem}

Zeros are counted in \eqref{eq:caro}, which is why exact zero ReLU gates and exact
score ties survive our enumeration. This is the multi-valued refinement of the
classical bound of \citet{warren1968lower}.

\begin{theorem}[van der Waerden permanent bound; \citealp{gurvits2005hyperbolic}, Theorem 2.1, diagonal specialisation]\label{ext:gurvits}
If $B$ is a $q\times q$ non-negative doubly stochastic matrix then
\begin{equation}\label{eq:vdw}
  \per(B)\ \ge\ q!/q^q .
\end{equation}
\end{theorem}

The bound \eqref{eq:vdw} was conjectured by van der Waerden and first proved by
\citet{falikman1981proof}, with the uniqueness statement due to
\citet{egorychev1981solution}; we quote the hyperbolic-polynomial proof of
\citet{gurvits2005hyperbolic} because its diagonal specialisation yields exactly
the inequality we use.

\begin{lemma}[Hilbert-space sampling; \citealp{pisier1981remarques}]
\label{ext:maurey}
Let $f=\sum_i v_i$ be a finite sum in a Hilbert space and
$S=\sum_i\lVert v_i\rVert$.  For every integer $K\ge1$ there is an
average of $K$ rescaled atoms with squared error at most $S^2/K$.
For $S>0$, one may sample $i$ with probability $\lVert v_i\rVert/S$
and use the atom $Sv_i/\lVert v_i\rVert$, omitting zero terms.
\end{lemma}

\begin{proof}
Each sampled atom has mean $f$ and squared norm $S^2$.  The expected
squared error of the average of $K$ independent samples is
$(S^2-\lVert f\rVert^2)/K\le S^2/K$, so some average attains this bound.
If $S=0$, use the zero average.
\end{proof}

\begin{theorem}[Bounded sampling; \citealp{hoeffding1963probability}]
\label{ext:hoeffding}
If $Y_1,\ldots,Y_L$ are independent variables in $[a,b]$ with $a<b$, then for $t>0$,
\begin{equation}
  \Prob\!\left[\frac1L\sum_{i=1}^L(Y_i-\Exp Y_i)\ge t\right]
  \le\exp\!\left(-\frac{2Lt^2}{(b-a)^2}\right).
\end{equation}
The same bound, including the lower-tail version, holds for a uniform
sample without replacement from a finite population in $[a,b]$, centred
at its population mean.  Applied to cell indicators, it yields the
hypergeometric estimate \eqref{eq:hypergeom}.
\end{theorem}

\begin{proposition}[Finite-group character identities;
{\citealp[Sections~3.1--3.5]{etingof2011representation}}]
\label{ext:characters}
For a finite abelian group $G$, its unit-modulus characters form a group
$\widehat G$ of order $|G|$, satisfy $\chi(u+v)=\chi(u)\chi(v)$, and
obey $\sum_{\chi\in\widehat G}\chi(g)=|G|\mathbf1\{g=0\}$.
For an arbitrary finite group, with irreducible unitary representations
$\pi$ of dimensions $d_\pi$ and characters $\chi_\pi$, the regular
representation identity is
\begin{equation}
  \sum_{\pi\in\widehat G}d_\pi^2=|G|,
  \qquad
  \sum_{\pi\in\widehat G}d_\pi\chi_\pi(g)=|G|\mathbf1\{g=e\}.
\end{equation}
\end{proposition}

%% file: appendix/02-additional-experiments.tex
\section{Experimental Protocols and Supplementary Results}
\label{app:experiments}

This appendix specifies the finite-training protocol, gives the complete
run-level results behind \Cref{fig:orbit-representation,fig:open-region},
and records the supplementary analyses and controls; the group-family runs
of \Cref{tab:group-family-endpoints} are part of the order-64 structure
sweep audited in \Cref{app:reproducibility}.  The experiments train the
model class of the theory, the pair-embedding one-hidden-layer ReLU network
of \Cref{sec:model-endpoint} with the $L_2$-penalised biases of
\Cref{app:bias}, for a fixed number of updates and report one predictor per
run; the theory concerns the complete minimum-norm endpoint set with the
dimensions and mask law specified in \Cref{sec:model}.

\subsection{Protocol}
\label{app:protocol}

\paragraph{Model and optimisation.}
The network (``PairMLP'', named for its pair embedding) has separate
trainable row and column embedding tables $E_U,E_V$ in $\Rbb^{128}$, one
width-$512$ ReLU hidden layer acting on the embedded pair, and a linear
$q$-class head, with biases on the hidden and output layers; this is the
architecture of \eqref{eq:network} with the biases of
\eqref{eq:biased-network}.  Its dimensions are fixed across orders,
whereas the theory uses $d_q=2q+1$ and $h_q=q^2$; the width-independent
notch bound still requires a feasible carrier to bound endpoint costs
(\Cref{app:discussion-atlas}).  The common
design that shares one embedding
table and concatenates the two embedded inputs is the special case
$E_U=W_1E$, $E_V=W_2E$ of the pair embedding.  Cross-entropy is minimised by
full-batch momentum SGD for $80{,}000$ updates (learning rate $0.03$,
momentum $0.9$, weight decay $0.001$ on embeddings, weights, and biases) in
float64.

\paragraph{Masks.}
The $q=97$ orbit comparison uses one uniform random fixed-cardinality split
with $2823$ observed cells out of $9409$; the implementation rounds
$0.30q^2$, while \eqref{eq:mask} uses a floor.  The $q=64$ experiments use
label-stratified masks, with the same number of observed cells from each
colour.  The realised sizes are
\begin{equation}
\begin{aligned}
  832/4096&=20.3125\% &&(\text{nominal }20\%),\\
  1216/4096&=29.6875\% &&(\text{nominal }30\%),\\
  1664/4096&=40.625\% &&(\text{nominal }40\%).
\end{aligned}
\end{equation}
For the coverage experiment, each colour has one random ordering and the masks
are prefixes of that ordering, so the three levels are nested.

\paragraph{Seeds and reporting units.}
A task seed selects a table, product residual, or coordinate relabelling; a
training seed selects the initialisation and optimisation repeat.  Every
point in the main figures is one training run; saturated markers and lines
report medians.

\paragraph{Relative transition.}
The transition statistic locates sustained training memorisation, estimates
a pre-generalisation test plateau and a stable terminal plateau, and counts a
run when its test accuracy crosses the midpoint of that run's
plateau-relative gain.  It permits terminal accuracy below one; absolute
endpoint accuracy is reported separately, and $0.99$ is a reporting cutoff
unrelated to the threshold $\tau$.

\subsection{Addition coverage across orders}
\label{app:addition-coverage}

The archived addition experiments vary the order $q$ and the observed fraction
$\alpha$.  At each order we sweep $\alpha$ on a grid of resolution $0.01$ and
record $\widehat\alpha_q$, the smallest tested fraction at which all five
optimisation seeds reach exact held-out accuracy $1.000$ under the protocol of
\Cref{app:protocol}.  Each grid point costs one training run per seed, so the
sweep covers four orders, $q\in\{47,79,113,211\}$; \Cref{tab:addition-coverage} is the
complete record.

\begin{table}[ht]
  \centering
  \small
  \caption{Empirical observation threshold for modular addition.
  $\widehat\alpha_q$ is the smallest fraction on a grid of resolution $0.01$
  at which all five optimisation seeds reach exact held-out accuracy $1.000$.
  Cell counts are the nominal $\lfloor\widehat\alpha_q q^2\rfloor$ of
  \eqref{eq:mask}; the implementation rounds, so realised sizes can differ by
  one cell.}
  \label{tab:addition-coverage}
  \begin{tabular}{@{}lcccc@{}}
    \toprule
    order $q$                                            & $47$   & $79$   & $113$   & $211$   \\
    \midrule
    $\widehat\alpha_q$                                   & $0.33$ & $0.25$ & $0.23$  & $0.20$  \\
    grid cells $q^2$                                     & $2209$ & $6241$ & $12769$ & $44521$ \\
    observed cells $\lfloor\widehat\alpha_q q^2\rfloor$  & $728$  & $1560$ & $2936$  & $8904$  \\
    \bottomrule
  \end{tabular}
\end{table}

The measured fraction falls from $0.33$ at $q=47$ to $0.20$ at $q=211$ while
the grid grows by a factor of about twenty.  Four points at resolution
$0.01$, with successive decrements of $0.08$, $0.02$, and $0.03$, fix a
decline over the tested window and nothing about its shape or its
monotonicity; the all-seeds criterion is one-sided, so each entry is an upper
bound on the fraction at which this protocol typically succeeds.  The
quantity measured is a
finite-training threshold under one optimiser, horizon, and architecture; the
sequence $\alpha_q$ of \Cref{rem:vanishing-observation} is an existence statement about
the full minimum-norm endpoint set with no rate.  The table is not used in
any proof.

\subsection{The order-97 orbit comparison}
\label{app:orbit-details}

The hidden-addition experiment samples ten independent triples of row, column,
and symbol permutations of one cyclic-addition table.  Each representation is
trained under three optimisation seeds, giving 30 runs.  The Latin control
uses eight task seeds from the sequential random-perfect-matching generator and
three optimisation seeds, giving 24 runs.  Two further planned control seeds were never completed and are excluded.
A one-way random-effects decomposition of the transition midpoint over the
ten task seeds with three optimisation seeds each attributes about $7\%$ of
the variance to the task seed and about $93\%$ to the optimisation seed
(variance-components estimate; the fixed-effects $\eta^2$ is $0.36$, and
the task-seed effect is not significant, $F(9,20)=1.22$, $p\approx0.33$;
between-task medians span $15{,}000$--$20{,}000$ updates, mean within-task
range $3{,}500$).  Midpoints are read on the $1{,}000$-update checkpoint
grid.

For the polynomial statistics illustrated in \Cref{fig:intro}, we reconstruct the
unique reduced polynomial over $\mathbb F_{97}$ by inverting the two
Vandermonde evaluation maps.  Table~\ref{tab:poly-reconstruction} reports the
result.  The hidden representatives occupy the same colouring orbit as
$X+Y$, while their observed-coordinate support is as dense as that of the
sequential Latin controls.

\begin{table}[t]
  \centering
  \small
  \caption{Reduced-polynomial reconstruction at $q=97$.  Support counts are
  out of $97^2=9409$ coefficients.}
  \label{tab:poly-reconstruction}
  \begin{tabular}{@{}lcccc@{}}
    \toprule
    representation & tasks & nonzero coefficients & $(\deg_X,\deg_Y)$ & total degree \\
    \midrule
    visible addition & 1  & $2$                 & $(1,1)$   & $1$ \\
    hidden addition  & 10 & $9108$--$9134$      & $(95,95)$ & $190$ \\
    sequential Latin & 8  & $9107$--$9127$      & $(95,95)$ & $190$ \\
    \bottomrule
  \end{tabular}
\end{table}

\subsubsection{Machine-readable coefficient archive}
\label{app:polynomial-coefficient-archive}
For every unique task instance in the order-$97$ comparison, we archive the
unique reduced coefficient matrix over $\mathbb F_{97}$ once per task seed.
Optimisation repeats share the same target table and therefore do not duplicate
the coefficient file.  Each sparse CSV contains the triples
$(i,j,a_{ij})$ for all nonzero terms $a_{ij}X^iY^j$, sorted first by $i$ and
then by $j$.  The archive is \texttt{materials/polynomial-coefficients/} in the
submission bundle.  It also contains \texttt{manifest.csv}, which records
the source experiment, support statistics, degree statistics, relative file
path, and full SHA-256 digest.  The 19 archived task instances comprise visible
addition, ten hidden-addition relabellings, and eight sequential-Latin controls.

\input{appendix/generated/polynomial-coefficient-index}

The median hidden-addition support is $9124/9409=96.971\%$.  Its 30 endpoint
accuracies have median $1.000$ and range $0.969025$--$1.000$; all 30 runs meet
the relative-transition criterion.  The 24
sequential-Latin endpoints have median $4.56\times10^{-4}$ and range
$0$--$1.37\times10^{-3}$, with 0/24 relative transitions.  Both constructions are exactly Latin; the sequential generator defines a
reproducible ensemble distinct from $\Unif(\Lq)$.

\subsection{Complete endpoint summaries for the open-region experiments}
\label{app:experiment-tables}

The designated-factor experiment (left panel of \Cref{fig:open-region}) uses five $q=64$
families, three task seeds per family, and two training seeds per task.  Table~\ref{tab:fixed-q2-full}
lists the final exact accuracy and the prespecified binary factor quotient.  The
$R_{64}$ family has no designated binary component.

\begin{table}[t]
  \centering
  \small
  \caption{Final medians at nominal $30\%$ label-stratified coverage; six runs per
  family.}
  \label{tab:fixed-q2-full}
  \begin{tabular}{@{}lcc@{}}
    \toprule
    family & exact accuracy & prespecified binary factor quotient \\
    \midrule
    $C_2\times C_{32}$            & $0.9094$  & $0.9977$ \\
    $C_2\times C_4\times R_8$    & $0.5938$  & $0.9922$ \\
    $C_2^2\times R_{16}$          & $0.0165$  & $0.9823$ \\
    $C_2\times R_{32}$            & $0.00330$ & $0.9816$ \\
    $R_{64}$                       & $0.00069$ & --- \\
    \bottomrule
  \end{tabular}
\end{table}

The run-level binary factor-quotient ranges for the two hardest structured families
are $0.456$--$0.989$ for $C_2^2\times R_{16}$ and
$0.478$--$0.984$ for $C_2\times R_{32}$.  This variation motivates the
run-level markers in the left panel of \Cref{fig:open-region}.

Table~\ref{tab:nested-mask-full} gives all median resolutions used in the
matched nested-mask comparison.  Every entry aggregates the same two task
seeds and two training seeds, giving four matched triplets per family.

\begin{table}[t]
  \centering
  \small
  \caption{Final held-out accuracy under nested masks.  A dash denotes a
  factor quotient absent from the factor chain.}
  \label{tab:nested-mask-full}
  \begin{tabular}{@{}llccc@{}}
    \toprule
    family & resolution & $20\%$ & $30\%$ & $40\%$ \\
    \midrule
    \multirow{4}{*}{$C_2^3\times R_8$}
      & binary factor quotient & $0.948$ & $0.999$ & $1.000$ \\
      & $4$-way factor quotient & $0.846$ & $0.998$ & $1.000$ \\
      & $8$-way factor quotient & $0.398$ & $0.998$ & $1.000$ \\
      & exact             & $0.0169$ & $0.882$ & $0.985$ \\
    \midrule
    \multirow{3}{*}{$C_2^2\times R_{16}$}
      & binary factor quotient & $0.495$ & $0.982$ & $0.992$ \\
      & $4$-way factor quotient & $0.344$ & $0.960$ & $0.984$ \\
      & exact             & $0.00184$ & $0.0189$ & $0.376$ \\
    \bottomrule
  \end{tabular}
\end{table}

\subsection{Factor-quotient lattice and intervention analyses}
\label{app:mechanism-experiments}

\Cref{sec:open-region} evaluates prespecified product-factor quotients at the final checkpoint.
The supplementary mechanism audit expands the evaluation to every checkpoint
and, for elementary $2$-group factors, to all factor-quotient directions of a fixed
order.  Eighteen runs contribute nine saved checkpoints each.  The audit
contains 550 factor-quotient-lattice records, 148 probe/intervention records, and 32
matched random-partition controls.

Two observations refine the endpoint picture.  First, different optimisation
runs can select different isomorphic binary factor-quotient directions.  In one
$C_2^2\times R_{16}$ run, for example, prespecified binary accuracy is $0.4559$
while an alternate binary factor quotient reaches $0.9858$.  Factor-quotient
recovery is therefore naturally evaluated over the available projection
lattice.  Second, factor-quotient labels are linearly
readable at every one of the 17 checkpoints selected for intervention, and
removing the classifier component associated with the recovered factor quotient
produces a substantially larger behavioural change than matched random
subspaces.  Removing one ridge-probe direction often has little effect; the
distributed classifier component is the more faithful intervention target.

These analyses concern representation dynamics along the trajectory; their
role here is to confirm that the factor-quotient accuracies in
\Cref{fig:open-region} reflect structured label resolution rather than the
contribution of exact hits to a coarse score.

\subsection{Supplementary families and negative controls}
\label{app:supplementary-experiments}

Several additional sweeps delimit the empirical interpretation.
\begin{itemize}[leftmargin=2.0em,itemsep=3pt,topsep=2pt]
  \item A product ladder from $C_{64}$ through $C_8\times R_8$,
        $C_4\times R_{16}$, $C_2\times R_{32}$, and $R_{64}$ shows a gradual
        loss of exact completion as the sequential residual factor grows.  The
        $C_8\times R_8$ endpoints range from $0.0427$ to $0.9924$, providing a
        finite-order non-group high-generalisation example.
  \item Sparse perturbations of the cyclic table yield a finite crossover at
        corruption levels around $2$--$5\%$.  Forty valid runs remain after
        removing task-seed collisions.  These endpoints quantify robustness to
        leaving the group orbit and carry no asymptotic phase label.
  \item A bank of nine named polynomial rules contains 81 runs.  Six families
        show relative transitions in all nine repeats; three cross-term
        variants show none.  The variant names used in the experiment archive
        are kept separate from the theoretical term \textsf{mixed}.
  \item A second stratified split reproduces the endpoint contrast between
        $C_{64}$ and $R_{64}$ in eight additional runs.
  \item Under the same SGD/weight-decay protocol, sixteen Transformer runs
        do not reach sustained training memorisation within the horizon and
        are therefore not comparable endpoints.
\end{itemize}
Timing-only weight-decay and angular-clock analyses remain in the experiment
archive as analyses of trajectory delay.

\subsection{Completeness and reproducibility audit}
\label{app:reproducibility}

The order-64 structure sweep contains 72/72 completed runs and the
factor-quotient sweep contains 64/64 completed runs, all reaching $80{,}000$
updates.  Stored
manifests and checkpoint hashes pass the archive checks.  The order-97 Latin
control contributes the 24 completed runs specified in
\Cref{app:orbit-details}; denominators use completed runs only.

Archived run identities record a base commit and a dirty-worktree flag, so
exact reproduction requires the frozen source snapshot, which the submission
archive preserves together with the run manifests, analysis script, endpoint
CSVs, figure inputs, and the finite-field polynomial reconstruction table.

%% file: appendix/generated/polynomial-coefficient-index.tex
\begin{table*}[t]
  \centering
  \scriptsize
  \setlength{\tabcolsep}{4pt}
  \caption{Index of the machine-readable sparse coefficient archives. Each CSV stores the nonzero triples $(i,j,a_{ij})$ over $\mathbb F_{97}$.}
  \label{tab:polynomial-coefficient-archive}
  \begin{tabular}{@{}lrrrrrl@{}}
    \toprule
    family & seed & nonzero & support (\%) & degrees & total & file \\
    \midrule
    visible addition & 0 & 2 & 0.021 & $(1,1)$ & 1 & \texttt{visible\_addition\_task\_000.csv} \\
    hidden addition & 0 & 9118 & 96.907 & $(95,95)$ & 190 & \texttt{hidden\_addition\_task\_000.csv} \\
    hidden addition & 1 & 9133 & 97.067 & $(95,95)$ & 190 & \texttt{hidden\_addition\_task\_001.csv} \\
    hidden addition & 2 & 9121 & 96.939 & $(95,95)$ & 190 & \texttt{hidden\_addition\_task\_002.csv} \\
    hidden addition & 3 & 9109 & 96.812 & $(95,95)$ & 190 & \texttt{hidden\_addition\_task\_003.csv} \\
    hidden addition & 4 & 9134 & 97.077 & $(95,95)$ & 190 & \texttt{hidden\_addition\_task\_004.csv} \\
    hidden addition & 5 & 9126 & 96.992 & $(95,95)$ & 190 & \texttt{hidden\_addition\_task\_005.csv} \\
    hidden addition & 6 & 9131 & 97.045 & $(95,95)$ & 190 & \texttt{hidden\_addition\_task\_006.csv} \\
    hidden addition & 7 & 9127 & 97.003 & $(95,95)$ & 190 & \texttt{hidden\_addition\_task\_007.csv} \\
    hidden addition & 8 & 9122 & 96.950 & $(95,95)$ & 190 & \texttt{hidden\_addition\_task\_008.csv} \\
    hidden addition & 9 & 9108 & 96.801 & $(95,95)$ & 190 & \texttt{hidden\_addition\_task\_009.csv} \\
    sequential Latin & 0 & 9121 & 96.939 & $(95,95)$ & 190 & \texttt{sequential\_latin\_task\_000.csv} \\
    sequential Latin & 1 & 9115 & 96.875 & $(95,95)$ & 190 & \texttt{sequential\_latin\_task\_001.csv} \\
    sequential Latin & 2 & 9107 & 96.790 & $(95,95)$ & 190 & \texttt{sequential\_latin\_task\_002.csv} \\
    sequential Latin & 3 & 9118 & 96.907 & $(95,95)$ & 190 & \texttt{sequential\_latin\_task\_003.csv} \\
    sequential Latin & 4 & 9119 & 96.918 & $(95,95)$ & 190 & \texttt{sequential\_latin\_task\_004.csv} \\
    sequential Latin & 5 & 9125 & 96.982 & $(95,95)$ & 190 & \texttt{sequential\_latin\_task\_005.csv} \\
    sequential Latin & 6 & 9127 & 97.003 & $(95,95)$ & 190 & \texttt{sequential\_latin\_task\_006.csv} \\
    sequential Latin & 7 & 9127 & 97.003 & $(95,95)$ & 190 & \texttt{sequential\_latin\_task\_007.csv} \\
    \bottomrule
  \end{tabular}
\end{table*}

%% file: appendix/03-discussion.tex
\section{Discussion}
\label{app:discussion}

The first two subsections discuss the stability of the all-good proof and
the relation between endpoints and training timescales.  The third turns
to task structure: it extracts a colouring-independent carrier criterion,
proves the all-good conclusion for every finite abelian group family, and
examines the necessity conjecture stated in \Cref{conj:carrier-necessity}.

\subsection{What the all-good proof rests on}
\label{app:discussion-atlas}

\paragraph{The competition.}
Stripped of its construction, the proof of \Cref{thm:main}\ref{it:main-add}
is the competition \eqref{eq:competition} between a counting term and a
deviation term.  $\mathcal H_q$ is the deterministic atlas of
\eqref{eq:atlas}: the enumeration, fixed before the mask is drawn, of every
cell set on which a low-norm score table can be confidently wrong while
remaining consistent on the observed cells.  The exponent $cN$ of
\eqref{eq:hypergeom} is the large-deviation price a positive-density set pays
to meet a uniform mask of density $\alpha$ in too few cells; against it,
$\log\card{\mathcal H_q}$ counts how many harmful patterns an adversary who
sees the mask may choose between.  The proof succeeds because
\eqref{eq:atlas} gives $\log\card{\mathcal H_q}=o(N)$, so the deviation term
dominates and every minimum-norm response is forced to generalise.  The
carrier of \Cref{lem:addition-carrier}, the SVD split, Maurey sampling, and
the sign-pattern counting exist only to make $\log\card{\mathcal H_q}$
subquadratic while keeping the atlas mask-independent.  That the atlas is
fixed first is why the conclusion covers the whole solution set: the
adversary's freedom to choose after seeing $\Omega$ has already been paid
for.

\paragraph{Stability in width.}
\Cref{lem:one-sided-notch} quantifies over tables of \emph{any} finite width,
and the atlas dimension of \eqref{eq:par-count},
$D_q^{\mathrm{par}}\le2qm+mK_0+5qK_1=O(q^{3/2}+r_qB_q^2/\sqrt q+q)$, contains
no width variable: $m=\lceil\sqrt q\rceil$ and the atom budgets $K_0,K_1$ of
\eqref{eq:budgets} depend only on $q$, the norm budget $B_q$, and the
schedule $r_q$.  Maurey sampling replaces a sum over arbitrarily many hidden
units by $K_0$ sampled atoms at a controlled error, so width enters the
argument only through the product $\frob A\nucnorm Z$.  Once the width
suffices to realise the carrier ($O(\log q)$ units, padded to $N$ in the
proof), further width changes the reachable norm but not the counting.

\paragraph{Fragility in depth.}
There is no depth-$L$ analogue, and the obstruction is structural: all three
ingredients of \eqref{eq:atlas} are tied to a single hidden layer.  The
variational identity behind $\Jred$ reduces a two-factor $L_2$ objective to
$\frob A\nucnorm Z$, with no comparable surrogate for a composition of $L$
factors; Maurey sampling compresses a sum of atoms, whereas depth composes
them, so the Gram bound \eqref{eq:gram} that caps each normalised gate at
$2q$ does not apply to inner layers; and the sign-pattern count treats
preactivations of degree $2$ in the core parameters, whereas at depth $L$
they become piecewise polynomial of degree growing with $L$.  Each of these
degrades the entropy estimate in the same direction.  This says the present
proof does not extend, not that deep networks fail; a depth-robust form of
\eqref{eq:competition} is the natural next target.

\paragraph{The empirical record on depth.}
\citet{fanDeepGrokking2024} train $12$-layer MLPs, find deep networks at least
as susceptible to delayed generalisation as shallow ones, and find that
internal feature rank, not weight norm, tracks the transition;
\citet{manirSystematicEmpiricalGrokking2026} report a non-monotone depth
effect on modular addition, with depth-$4$ MLPs failing to grok and
depth-$8$ residual networks recovering generalisation.  Both point to
architectures in which upper layers are forced to track a representation
built below, which is also the situation the atlas argument can cover: a
single controlled map over a given feature table.  A norm bound over a
solution set and a rank statistic along a trajectory answer different
questions, and a trajectory-level account of depth would likely be organised
around a different quantity than the one controlled here.

\subsection{Timescales, and the extent to which they decouple}
\label{app:discussion-timescale}

\paragraph{Scope.}
The phase of \eqref{eq:phase-quantifiers} is a property of the solution set
$\Sq(\Dset)$ and carries no time index; every statement in \Cref{sec:main}
would read identically if the optimiser reached its endpoint immediately.
The transition statistic of \Cref{app:protocol} is a reporting device, not
an object the theory predicts.

\paragraph{What our measurements show.}
The order-$97$ orbit comparison is a clean instrument for the separation:
all ten hidden-addition tasks are relabellings of one addition table, so
\Cref{thm:quotient-invariance} makes the phase constant along the orbit, and
the trained predictors have concentrated final accuracies
($30/30$ transitions, median $1.000$, range $0.969$--$1.000$).
The delay is not constant: midpoints span $13{,}000$--$21{,}000$ updates,
with about $93\%$ of their variance attributable to the optimisation seed
(\Cref{app:orbit-details}).  Within one orbit the phase is invariant and
the delay is dominated by optimisation randomness.  This does not establish
that endpoint and delay are independent: every family we sweep either fixes
the task and varies the optimiser, or varies the task and reports endpoints;
none varies task structure while measuring delay under a comparable phase.

\paragraph{Where the coupling is visible.}
\citet{guBeyondProgressMeasures2025} distinguish the accuracy ceiling
from the training process leading to it, and their experiments measure
both the ceiling and transition time.  At
$p=97$ with a $30\%$ training fraction, they replace the random split by
excised training distributions---a square block $i,j\le k$ and a strip
$i\le t$---and record both the accuracy ceiling after the transition and the
epoch at which it occurs (\Cref{tab:gu-excision}).  Enlarging the excision
lowers the ceiling monotonically and shifts the transition epoch
non-monotonically; the widest strip removes the transition.  Under a random
split the same configuration transitions near epoch $5000$ and reaches
accuracy $1$.  One manipulation thus moves ceiling and delay together.

\begin{table}[ht]
  \centering
  \small
  \caption{Excised training distributions at $p=97$, reproduced from
  \citet{guBeyondProgressMeasures2025}.  U-Acc is the accuracy ceiling
  after the transition, GO-Epoch the epoch at which it occurs.}
  \label{tab:gu-excision}
  \begin{tabular}{@{}llcc@{}}
    \toprule
    excision & parameter & U-Acc (\%) & GO-Epoch \\
    \midrule
    \multirow{4}{*}{square $i,j\le k$}
      & $k=30$ & $99.04$ & $1500$ \\
      & $k=40$ & $84.50$ & $1700$ \\
      & $k=50$ & $77.39$ & $3400$ \\
      & $k=60$ & $57.31$ & $2200$ \\
    \midrule
    \multirow{3}{*}{strip $i\le t$}
      & $t=40$ & $89.89$ & $1300$ \\
      & $t=50$ & $80.10$ & $1600$ \\
      & $t=67$ & $7.56$  & none observed \\
    \bottomrule
  \end{tabular}
\end{table}

A ceiling below one, with runs halting at different accuracies, is not
evidence of the mixed phase: that phase concerns the exact minimum-norm set
(\Cref{rem:mixed-interval}), and no training curve reports it.

\paragraph{How this sits against our framework.}
Both of their manipulations fall outside our hypotheses in instructive ways.
The excisions keep the colouring at $\Cadd$ and replace the mask law: our
$\nu_q$ is the uniform fixed-cardinality law of \eqref{eq:mask}, and
uniformity is what makes $\card{H\cap\Omega}$ hypergeometric in
\eqref{eq:hypergeom} and so supplies the deviation term of
\eqref{eq:competition}.  A square or strip excision is a mask that a
positive-density harmful set can evade at no exponential cost, so the
degradation of the ceiling is consistent with, and delimits, what the
uniform-mask assumption does.  Their second manipulation changes the rule to
$f(i,j)=i^2+ij+j^2$, whose ceiling stays below $1$ under every split and on
which a two-layer network stalls near $0.01$ while a transformer retains
partial generalisation; for fixed $i$ that map is quadratic in $j$, so the
rule is not Latin and lies outside $\Lq$.  What survives the translation is
the separation of the accuracy ceiling from the process leading to it.
Our phase formalises an endpoint question by quantifying over the whole
minimum-norm set under uniform masks; their experiments measure trained
transformers under designed masks and expose the role of the training
distribution.  This comparison does not establish that the attainable
accuracy is independent of the optimiser.  The coupling between that
ceiling and the delay is visible in their data and absent from ours; it is
an open direction.

\subsection{What the all-good proof needs from the colouring}
\label{app:colouring-carriers}

\Cref{app:discussion-atlas} asked what the proof is stable to as the model
varies.  Here we vary the colouring while keeping the bias-free model,
the product norm \eqref{eq:Jred}, and the mask law \eqref{eq:mask} fixed.
Only the carrier construction uses addition; the counting argument that
turns a carrier into an all-good statement does not.

\subsubsection{A colouring-independent notch and a carrier criterion}
\label{app:carrier-criterion}

\begin{lemma}[One-sided notch for any colouring]
\label[lemma]{lem:colouring-notch}
Fix $\alpha,\beta\in(0,1)$ and a deterministic sequence $C_q\in\Cq$.
Let $B_q$ satisfy $B_q^2\log q/q^{5/2}\to0$.  With probability $1-o(1)$
over $\Omega\sim\nu_q$, there is no finite-width table
$D(x)=A\relu(Z\xicell_x)$ with $\frob A\nucnorm Z\le B_q$ such that
\begin{enumerate}[label=(\arabic*),leftmargin=2.2em,itemsep=1pt,topsep=1pt]
  \item $D_{C_q(x)}(x)-D_a(x)\ge0$ for every $x\in\Omega$ and
        $a\ne C_q(x)$; and
  \item at least $\beta N$ cells admit a rival $a\ne C_q(x)$ with
        $D_a(x)-D_{C_q(x)}(x)\ge1$.
\end{enumerate}
The failure bound is uniform in the choice of the deterministic colouring.
\end{lemma}

\begin{proof}
In \Cref{app:notch-proof}, replace $y(x)=\Cadd(x)$ by $y(x)=C_q(x)$.
The Gram matrix in Step~0 depends only on the grid and on $\xicell$.
The low-rank split, sampling, and cover codes of Steps~1--3 use only the
score table and its norm budget.  In Step~4, the labels select the truth
coordinate in $N(q-1)$ polynomial comparisons, without changing their
degrees or parameter counts.  The resulting atlas may depend on $C_q$,
but its cardinality bound does not, and it is fixed before the mask is
drawn.  Step~5 uses only the size of a set and the uniform mask law.
Thus the same bounds prove the claim for each fixed colouring, uniformly
in that colouring.  This is not a simultaneous event over all colourings.
\end{proof}

\begin{proposition}[Carrier criterion]\label[proposition]{prop:carrier-criterion}
Let $C_q\in\Cq$ be deterministic.  Suppose that, for every sufficiently
large $q$, it admits a full-grid carrier
$\Ccar=A_{\rm car}\relu(Z_{\rm car}\xicell_{\cdot})$ of width at most
$h_q=N$, with $J_{{\rm car},q}:=\frob{A_{\rm car}}\nucnorm{Z_{\rm car}}$
and $\Gamma_q\ge1$ satisfying
\begin{equation}\label{eq:general-carrier}
\begin{gathered}
  1\le\Ccarc{C_q(x)}(x)-\Ccarc a(x)\le\Gamma_q
  \qquad(x\in\Xq,\ a\ne C_q(x)),\\
  \Gamma_q J_{{\rm car},q}
  =o\!\left(\frac{q^{5/4}}{\sqrt{\log q}}\right).
\end{gathered}
\end{equation}
Then for every fixed $\alpha,\tau\in(0,1)$,
\begin{equation}\label{eq:carrier-allgood}
  \Aminus(C_q,\Omega)\xrightarrow{\mathrm P}1,
  \qquad \pall(C_q)\to1,
  \qquad \profile([C_q])\to\vertexV=(1,0,0).
\end{equation}
\end{proposition}

\begin{proof}
Zero padding makes the carrier feasible in \eqref{eq:Jred} on every mask,
so $\Jstar\le J_{{\rm car},q}$.  For any endpoint response $F$, take its
width-$N$ reduced representation from \Cref{lem:reduced-normalization}
and put $D=\Gamma_q F-\Ccar$.  On $\Omega$, the observed unit margin and
the upper carrier gap give $D_y-D_a\ge0$, where $y=C_q(x)$.  At every
held-out non-strict error, some $a\ne y$ has $F_a\ge F_y$, so the lower
carrier gap gives $D_a-D_y\ge1$.

Balance the two representations by positive homogeneity so that
$\frob{A_F}=\nucnorm{Z_F}=\sqrt{\Jstar}$ and
$\frob{A_{\rm car}}=\nucnorm{Z_{\rm car}}=\sqrt{J_{{\rm car},q}}$.
The stacked factors
\begin{equation}\label{eq:general-carrier-stack}
  \widetilde A=[\,\sqrt{\Gamma_q}A_F\ \ -A_{\rm car}\,],
  \qquad
  \widetilde Z=[\,\sqrt{\Gamma_q}Z_F;\ Z_{\rm car}\,]
\end{equation}
represent $D$, and the same estimate as \eqref{eq:JD} gives
\begin{equation}\label{eq:general-carrier-budget}
\begin{aligned}
  \frob{\widetilde A}\nucnorm{\widetilde Z}
  &\le\sqrt2\,(\Gamma_q\Jstar+J_{{\rm car},q})\\
  &\le\sqrt2\,(\Gamma_q+1)J_{{\rm car},q}=:B_q .
\end{aligned}
\end{equation}
Since $\Gamma_q\ge1$, condition \eqref{eq:general-carrier} is equivalent
to $B_q^2\log q/q^{5/2}\to0$.  The stack may exceed width $N$, which
\Cref{lem:colouring-notch} permits; the carrier itself must fit width
$N$ to justify the bound on $\Jstar$.
An endpoint with $\Acc_T<\tau$ supplies more than
$(1-\tau)(1-\alpha)N$ reverse-unit-gap cells.  Apply
\Cref{lem:colouring-notch} with
$\beta=\tfrac12(1-\tau)(1-\alpha)$.  Its exclusion event covers every
endpoint simultaneously, proving \eqref{eq:carrier-allgood}.
\end{proof}

The proposition is a sufficient condition.  Its contrapositive constrains
families that fail to be all-good; it does not say that a family without
such a carrier fails to generalise.  In particular, the observed optimum
$\Jstar$ can be cheaper than a chosen full-grid construction, and a
colouring-specific atlas can be smaller than our bound for the whole
norm ball.  The necessity direction in \Cref{conj:carrier-necessity}
is an additional hypothesis.

\subsubsection{Finite abelian group tables}
\label{app:abelian-carriers}

Let $G$ be a finite abelian group of order $q$, identify its elements with
$\qset$, and write its operation additively.  Its table
$C_G(u,v)=u+v$ is Latin.  Let $\widehat G$ be its character group, with
$q$ elements, and write
$\chi(t)=c_\chi(t)+i s_\chi(t)$, where
$c_\chi(t)^2+s_\chi(t)^2=1$.
We use the standard finite-group character identities of
\citet[Sections~3.1--3.5]{etingof2011representation}.

\begin{lemma}[Abelian carrier]\label[lemma]{lem:abelian-carrier}
There is an absolute constant $C_{\rm ab}$ such that, for every finite
abelian group $G$ of sufficiently large order $q$, the table $C_G$ has a
full-grid carrier of width $O(\log q)$, truth--rival gaps in $[1,9]$,
and $J_{{\rm car},q}\le C_{\rm ab}q\sqrt{\log q}$.
\end{lemma}

\begin{proof}
Keep $h,\epsilon,\rho,c,L$ and the ReLU product surrogate $P$ from
\Cref{app:carrier-proof}.  We give each change to that construction.

\paragraph{Character selection.}
For $e\ne0$, character orthogonality gives
$\sum_{\chi\in\widehat G}\chi(e)=0$.  Indeed, a character $\psi$ with
$\psi(e)\ne1$ exists, and multiplying the sum by $\psi(e)$ merely
permutes its terms.  Thus, if $\mathbf1$ denotes the trivial character,
\begin{equation}\label{eq:abelian-character-mean}
  \Exp_{\chi\sim\Unif(\widehat G\setminus\{\mathbf1\})}
  [c_\chi(e)]=-\frac1{q-1}.
\end{equation}
The real parts lie in $[-1,1]$.  Hoeffding's inequality and a union bound
over the $q-1$ nonzero elements therefore give, with the same choice of
$L=O(\log q)$, characters $\chi_1,\ldots,\chi_L$ such that
\begin{equation}\label{eq:abelian-character-cover}
  \frac1L\sum_{\ell=1}^L c_{\chi_\ell}(e)\le\rho
  \qquad(e\ne0).
\end{equation}
Sampling with replacement is allowed; no distinctness is needed.

\paragraph{Products, scores, and gaps.}
Multiplicativity gives the real and imaginary addition identities.  Set
\begin{align}
  \widetilde c_\ell(u,v)
  &=P(c_{\chi_\ell}(u),c_{\chi_\ell}(v))
    -P(s_{\chi_\ell}(u),s_{\chi_\ell}(v)),\nonumber\\
  \widetilde s_\ell(u,v)
  &=P(s_{\chi_\ell}(u),c_{\chi_\ell}(v))
    +P(c_{\chi_\ell}(u),s_{\chi_\ell}(v)).
  \label{eq:abelian-products}
\end{align}
By \eqref{eq:prod-approx}, these approximate
$c_{\chi_\ell}(u+v)$ and $s_{\chi_\ell}(u+v)$ to error at most $h^2/8$.
Define
\begin{equation}\label{eq:abelian-carrier-score}
  \Ccarc a(u,v)=\frac cL\sum_{\ell=1}^L
  \bigl(c_{\chi_\ell}(a)\widetilde c_\ell(u,v)
       +s_{\chi_\ell}(a)\widetilde s_\ell(u,v)\bigr).
\end{equation}
Its ideal version is
$\frac cL\sum_\ell\operatorname{Re}\chi_\ell(u+v-a)$, which is $c$
at the true label and at most $c\rho$ at a rival.  Since
$|c_\chi(a)|+|s_\chi(a)|\le\sqrt2$, the score error is at most
$c\sqrt2h^2/8$.  The margin calculation in \Cref{app:carrier-proof}
therefore gives exactly the same interval $[1,9]$.

\paragraph{Architecture and cost.}
Expanding $P$ using \eqref{eq:Qh} produces a constant number of hinges
per character.  Each hinge is $\pospart{r_u+s_v-t}$, with bounded row
and column potentials; splitting $t$ equally between them makes it a
single reduced ReLU unit.  The preactivation row has norm $O(\sqrt q)$,
and its output column has norm $O(\sqrt q/L)$, with constants independent
of $G$.  Consequently
\begin{equation}\label{eq:abelian-carrier-norms}
  \frob{A_{\rm car}}=O(\sqrt{q/L}),\qquad
  \nucnorm{Z_{\rm car}}\le\sqrt{O(L)}\frob{Z_{\rm car}}
  =O(L\sqrt q),
\end{equation}
giving $J_{{\rm car},q}=O(q\sqrt{\log q})$ and width $O(\log q)\le N$.
\end{proof}

\begin{proof}[Proof of \Cref{cor:abelian-allgood}]
\Cref{lem:abelian-carrier} gives
$\Gamma_qJ_{{\rm car},q}=O(q\sqrt{\log q})
=o(q^{5/4}/\sqrt{\log q})$.  Apply \Cref{prop:carrier-criterion} and
\Cref{thm:quotient-invariance}.
\end{proof}

Thus the all-good pole contains every abelian group family, beyond the
cyclic case of \Cref{thm:main}\ref{it:main-add}.  Non-isomorphic abelian
groups give distinct colouring orbits: an isotopy between two such group
tables can be written $h(u+v)=f(u)+g(v)$; substituting the identity in
each argument shows that $u\mapsto h(u)-h(0)$ is a group isomorphism.
Transposition changes no abelian table.  At orders $q=2^m$, partitions
of $m$ therefore supply an unbounded number of distinct abelian orbits
($11$ at $m=6$), all covered by the same asymptotic bound.  This is a
limit statement along families, not a claim that a finite-order profile
equals $\vertexV$ exactly.

\subsubsection{What the finite-order experiments can suggest}
\label{app:carrier-experiments}

The abelian rows $C_2^6$, $C_{64}$, and $C_8\times C_8$ in
\Cref{tab:group-family-endpoints} are finite members of families covered
by \Cref{cor:abelian-allgood}.  The corollary concerns the whole
minimum-norm endpoint set under $\nu_q$ as $q\to\infty$; the table
reports one trained predictor per run at $q=64$, under label-stratified
masks and a fixed horizon.  It neither tests the endpoint quantifier nor
orders the finite-training accuracies.

There is nevertheless a concrete simplification for elementary abelian
$2$-groups.  Every character is real with values in $\{\pm1\}$, so
$s_\chi=0$ and \eqref{eq:abelian-products} reduces to one product
$P(c_\chi(u),c_\chi(v))$ per character.  Its two arguments have sum and
difference in $\{-2,0,2\}$, interpolation nodes of $Q_h$, hence
$P(x,y)=xy$ exactly on these inputs.  The construction uses fewer hinges
and has no multiplication error.  These are improvements to this
explicit circuit, not a comparison of the minimum achievable carrier
norms across groups or a prediction of the $6/6$ exact completions for
$C_2^6$.

\subsubsection{Non-abelian tables and the missing cost estimate}
\label{app:nonabelian-carriers}

The transitions for $\mathrm{Dih}_{64}$ invite an extension beyond scalar
characters.  Every one-dimensional character of a non-abelian group
factors through its abelianisation.  It therefore equals $1$ on every
commutator, so no selection of such characters can satisfy
\eqref{eq:abelian-character-cover} at a non-identity commutator.  For the
order-$64$ dihedral group, the abelianisation has only four elements.

Write $\widehat G$ now for the irreducible unitary dual, let
$\pi\in\widehat G$ have dimension $d_\pi$, and set
$\chi_\pi=\operatorname{tr}\pi$.  The regular-representation identity
gives
\begin{equation}\label{eq:nonabelian-character-mean}
  \sum_{\pi\in\widehat G}\frac{d_\pi^2}{q}
  \frac{\chi_\pi(g)}{d_\pi}=\mathbf1\{g=e\}.
\end{equation}
The real part of the normalised character lies in $[-1,1]$, so sampling
with weights $d_\pi^2/q$ gives a separation argument analogous to the
abelian one.  The candidate ideal score for label $a$ uses
$\operatorname{Re}\chi_\pi(uva^{-1})/d_\pi$, rather than
$\chi_\pi(uv)$ alone.  In particular,
\begin{equation}\label{eq:nonabelian-score-products}
  \frac{\chi_\pi(uva^{-1})}{d_\pi}
  =\frac1{d_\pi}\sum_{i,j,k=1}^{d_\pi}
    \pi(u)_{ik}\pi(v)_{kj}\overline{\pi(a)_{ij}}.
\end{equation}
Each complex product of a row quantity and a column quantity can again be
expanded into real products and approximated using $P$.  But this direct
construction contains $d_\pi^3$ scalar complex products shared across
the label scores; counting only the $d_\pi^2$ terms in
$\operatorname{tr}(\pi(u)\pi(v))$ does not cost all those scores.
Moreover, multiplication errors accumulate with dimension, so the mesh
may need refinement, and the nuclear and output norms must be recomputed.
We leave these costs open.  Families with uniformly bounded irreducible
dimensions, including dihedral groups, are a concrete first target for
\Cref{prop:carrier-criterion}; growing dimensions require quantitative
control of both approximation and representation size.

\subsubsection{Rarity, and what a necessary carrier would mean}
\label{app:carrier-necessity}

Every colouring has a full-grid carrier of finite, indeed polynomial,
cost: the cell-indicator construction in \Cref{app:prelim-proofs} uses
$N$ units and bounded parameter entries.  This includes non-group Latin
tables and uniformly random Latin squares.  What is unknown for the
non-group products of \Cref{sec:open-region} is a construction cheap
enough to satisfy \eqref{eq:general-carrier} along an infinite family.

The uniform obstruction does give a precise rarity statement for cheap
carriers.  For a deterministic
$b_q=o(q^{5/4}/\sqrt{\log q})$, let $\mathcal K_q(b_q)\subseteq\Lq$
contain the colourings with a width-at-most-$N$ carrier whose gaps lie in
$[1,\Gamma]$ and whose cost satisfies $\Gamma J_{\rm car}\le b_q$ for
some $\Gamma\ge1$.  The proof of \Cref{prop:carrier-criterion}, using
$B_q=2\sqrt2 b_q$, gives $\pall(C)\ge1-o(1)$ uniformly on this set.
Thus \Cref{thm:main}\ref{it:main-unif} implies
\begin{equation}\label{eq:cheap-carrier-rarity}
  \frac{\card{\mathcal K_q(b_q)}}{\card\Lq}\,(1-o(1))
  \le\Pall(\Munif)\longrightarrow0.
\end{equation}
This rules out this cheap-carrier regime for almost all Latin colourings;
it rules out neither arbitrary carriers nor the existence of a good
endpoint in a mixed instance.

The sufficient budget in \eqref{eq:general-carrier} is not necessary.
The accompanying \emph{Carrier necessity counterexample} note constructs
a deterministic Latin family along even orders for which
$\Aminus\xrightarrow{\mathrm P}1$ at every fixed positive observation
fraction, yet the minimum of $\Gamma J_{\rm car}$ over all width-at-most-$N$
carriers with gaps in $[1,\Gamma]$ is $(1+o(1))q^{13/10}$.
Its ratio to $q^{5/4}/\sqrt{\log q}$ tends to infinity.
The construction uses independent four-cell Latin trades and the first
crossing of a carrier-cost threshold; a multiscale margin count still
establishes all-goodness.  Thus the counterexample separates the old
numerical budget from the broader carrier mechanism.  Its lower bound
concerns the joint cost $\Gamma J$, not the unit-margin cost without
$\Gamma$.

\paragraph{A mask-independent robust certificate.}
For a colouring $C$ and $0\le\varepsilon<1$, define
\begin{equation}\label{eq:robust-margin-cost}
 \mathcal R_q(C,\varepsilon)=\inf\left\{\frob A\nucnorm Z:
 \begin{array}{l}
 w\le N,\quad E\subseteq\Xq,\quad\card E\le\varepsilon N,\\
 F_{C(x)}(x)-F_a(x)\ge1\quad(x\notin E,\ a\ne C(x))
 \end{array}\right\},
\end{equation}
where $F=A\relu(Z\xicell_{\cdot})$.  The infimum ranges jointly over
the representation and the exceptional set $E$.  Both are chosen from
the deterministic task, independently of the mask.  No upper gap is
prescribed, and the architecture and width restriction remain those of
\Cref{sec:model}.

In this notation, \Cref{conj:carrier-necessity} asserts that if a
deterministic Latin sequence satisfies
$\Aminus(C_q,\Omega)\xrightarrow{\mathrm P}1$ for every fixed
$\alpha\in(0,1)$ under its corresponding mask law, then there exists a
deterministic sequence $\varepsilon_q\to0$ such that
\begin{equation}\label{eq:carrier-necessity-conjecture}
 \mathcal R_q(C_q,\varepsilon_q)=o(q^{3/2}).
\end{equation}
The premise includes every fixed accuracy threshold $\tau<1$, separately
at each fixed observation fraction; it does not assert a single event
uniform in $\alpha$.  The conclusion is equivalent to the existence of
the carriers in the main-text statement: near-minimisers of the
infimum suffice to select their representations and exceptional sets.

The upper-gap factor $\Gamma$ pays for the subtraction
$\Gamma F-\Ccar$ in the original proof, whereas a direct margin argument
can work with $F$ itself.  Removing it avoids treating that proof cost
as an intrinsic necessary complexity.  Allowing $o(N)$ exceptional
cells matches an accuracy limit that can ignore a vanishing fraction
of difficult cells.  Requiring all-goodness at every fixed positive
observation fraction avoids assuming an unproved transfer from one
fraction to all others.  These are independent changes to the converse,
not a claim that merely replacing the exponent $5/4$ by $3/2$ gives a
sharp characterisation.

The accompanying multiscale exploration gives a sufficient condition
using a full-grid unit-margin carrier of cost
$o(q^{3/2}/\log q)$, without an upper-gap factor.  It does not prove
\eqref{eq:carrier-necessity-conjecture}: all-goodness supplies neither
an upper bound on $\Jstar$ nor a quantitative lower bound on the positive
held-out margins.  Scaling a good endpoint to unit margin outside a
small exceptional set may therefore have uncontrolled cost.  Averaging
endpoints over masks does not resolve this problem, since accuracy does
not control the magnitude of rare negative gaps or the norm of a
combined representation.  An entropy upper bound also cannot be
reversed to force bad global minimisers above its sufficient budget.

Conversely, a robust carrier need not satisfy the constraints at
observed exceptional cells, so it does not by itself give
$\Jstar\le\mathcal R_q(C_q,\varepsilon_q)$.  Repairing an arbitrary
$o(N)$-cell set requires additional cost and width control, beyond the
four-cell trade argument.  The revised conjecture is therefore only
a proposed necessary condition, not an equivalence; neither necessity
of the exponent $3/2$ nor a critical logarithmic factor is established.